\documentclass[Afour,sageh,times]{sagej}

\usepackage{moreverb,url}
\usepackage{array}
\usepackage{textcomp}
\usepackage{stfloats}
\usepackage{url}
\usepackage{verbatim}
\usepackage{graphicx}
\usepackage{amsmath}
\usepackage{tabularx} 
\usepackage{amsthm}
\usepackage{bm}
\usepackage{subcaption}

\DeclareMathOperator*{\argmax}{arg\,max}

\newcommand{\trsp}{{\scriptscriptstyle\top}}
\newtheorem{example}{Example}
\newtheorem{theorem}{Proposition}

\newcommand{\mc}{\mathcal}
\newcommand{\mb}{\mathbb}

\usepackage[colorlinks,bookmarksopen,bookmarksnumbered,citecolor=red,urlcolor=red]{hyperref}

\newcommand\BibTeX{{\rmfamily B\kern-.05em \textsc{i\kern-.025em b}\kern-.08em
T\kern-.1667em\lower.7ex\hbox{E}\kern-.125emX}}

\def\volumeyear{2024}
\begin{document}

\runninghead{Xue et al.: Implicit Motor Adaptation}

\title{Robust Contact-rich Manipulation through Implicit Motor Adaptation}

\author{Teng Xue\affilnum{1, 2}, Amirreza Razmjoo\affilnum{1, 2}, Suhan Shetty\affilnum{1, 2} and Sylvain Calinon \affilnum{1, 2}}

\affiliation{\affilnum{1}Idiap Research Institute, Martigny, Switzerland\\
\affilnum{2}École Polytechnique Fédérale de Lausanne (EPFL), Lausanne, Switzerland}

\corrauth{Teng Xue, Idiap Research Institute, Rue Marconi 19, 1920
Martigny, Switzerland.}

\email{teng.xue@idiap.ch}

\begin{abstract}
Contact-rich manipulation plays an important role in daily human activities. However, uncertain physical parameters often pose significant challenges for both planning and control. A promising strategy is to develop policies that are robust across a wide range of parameters. Domain adaptation and domain randomization are widely used, but they tend to either limit generalization to new instances or perform conservatively due to neglecting instance-specific information. \textit{Explicit motor adaptation} addresses these issues by estimating system parameters online and then retrieving the parameter-conditioned policy from a parameter-augmented base policy. However, it typically requires precise system identification or additional training of a student policy, both of which are challenging in contact-rich manipulation tasks with diverse physical parameters.
In this work, we propose \textit{implicit motor adaptation}, which enables parameter-conditioned policy retrieval given a roughly estimated parameter distribution instead of a single estimate. We leverage tensor train as an implicit representation of the base policy, facilitating efficient retrieval of the parameter-conditioned policy by exploiting the separable structure of tensor cores. This framework eliminates the need for precise system estimation and policy retraining while preserving optimal behavior and strong generalization. We provide a theoretical analysis to validate the approach, supported by numerical evaluations on three contact-rich manipulation primitives. Both simulation and real-world experiments demonstrate its ability to generate robust policies across diverse instances.

Project website: \href{https://sites.google.com/view/implicit-ma}{https://sites.google.com/view/implicit-ma}.

\end{abstract}


\maketitle

\section{Introduction}

Robot manipulation usually involves multiple contact-rich manipulation primitives, such as \texttt{Push} and \texttt{Pivot}, leading to a long-horizon problem characterized by hybrid types of variables (discrete and continuous). The resulting combinatorial complexity poses significant challenges to most planning and control approaches. Instead of treating long-horizon manipulation as a whole, one way is to decompose the long-horizon process into several simple contact-rich manipulation primitives and then sequence them using PDDL planners \citep{mcdermott1998pddl, Xue24RSS, cheng2023league} or Large Language Models \citep{driess2023palm}. Since manipulation primitives are typically sequenced by naive symbolic planners that lack geometric or motion-level information, it is crucial to develop primitives that are robust to diverse instances with varying physical parameters, such as shape, mass, and friction coefficient. For example, once a \texttt{push} primitive is scheduled by a high-level symbolic planner, it should be capable of pushing objects with different physical properties from any initial configuration to their targets.

\begin{figure*}[h]
\centering
\begin{tabular}
{>{\centering\arraybackslash}m{2.4cm} |>{\centering\arraybackslash}m{2.2cm}>{\centering\arraybackslash}m{2.2cm}>{\centering\arraybackslash}m{2.2cm}>{\centering\arraybackslash}m{2.2cm}>{\centering\arraybackslash}m{2.2cm}}
 \includegraphics[width=2.4cm]{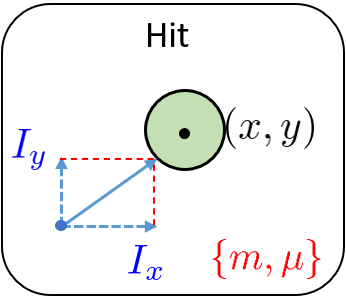} & \includegraphics[width=2.5cm]{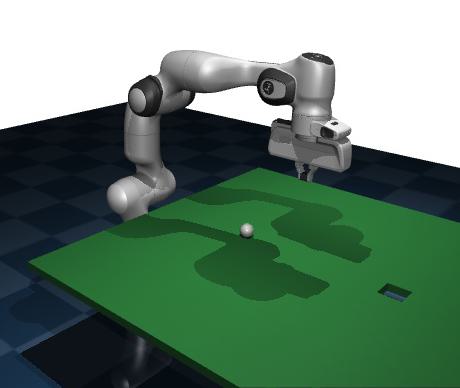} & \includegraphics[width=2.5cm]{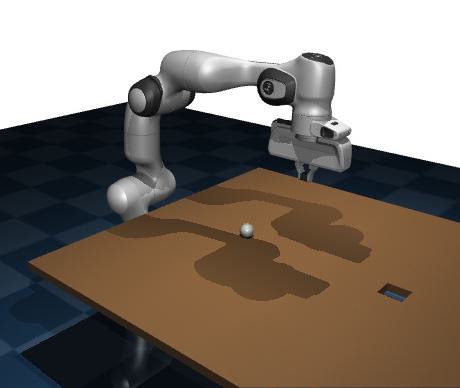} & \includegraphics[width=2.5cm]{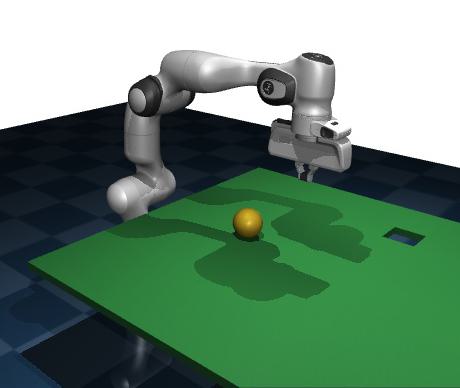} & \includegraphics[width=2.5cm]{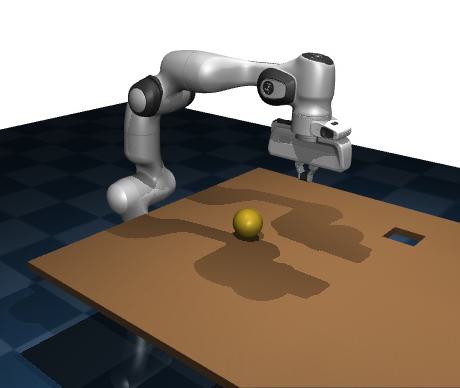} & \includegraphics[width=2.5cm]{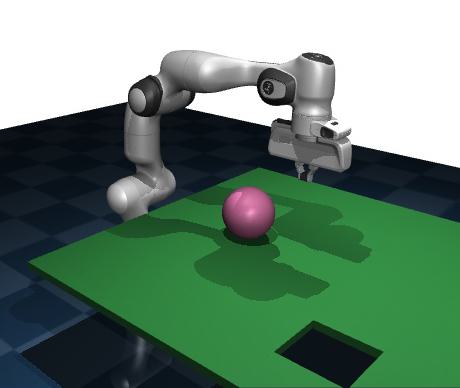} \\

\includegraphics[width=2.4cm]{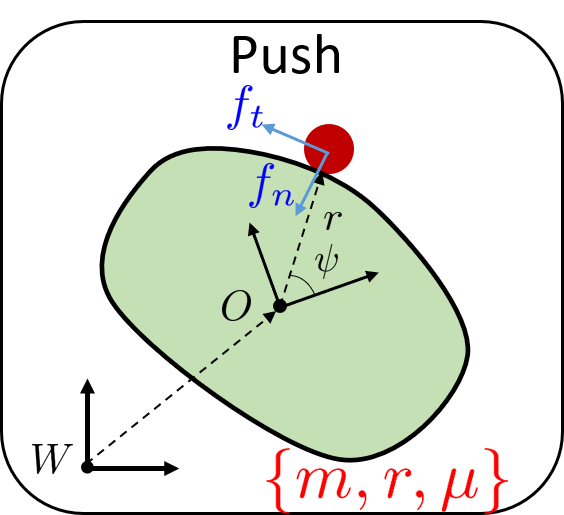} & \includegraphics[width=2.5cm]{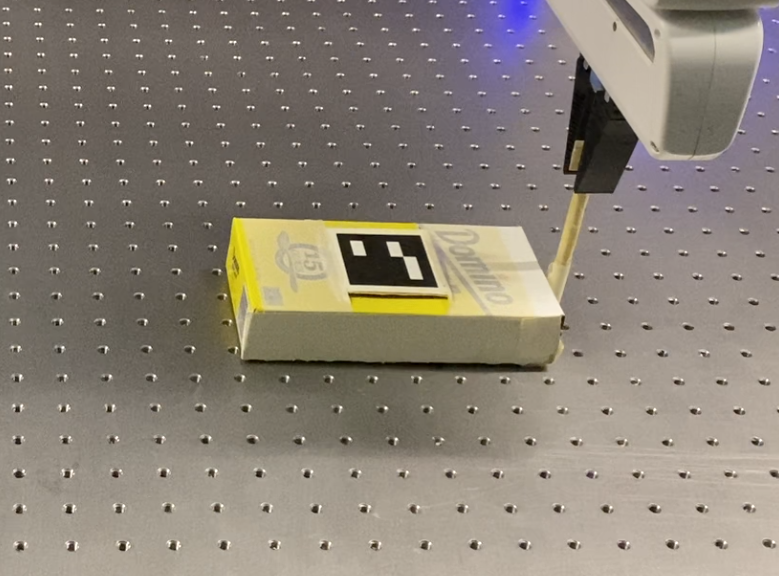} & \includegraphics[width=2.5cm]{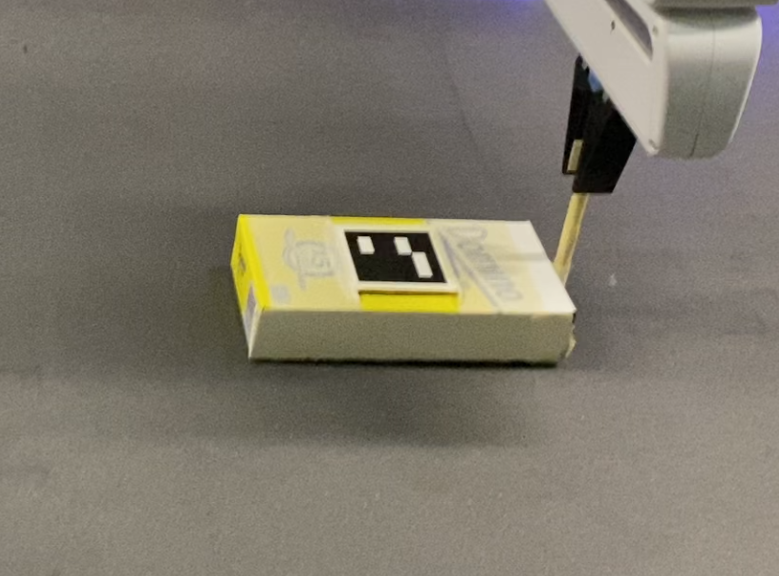} & \includegraphics[width=2.5cm]{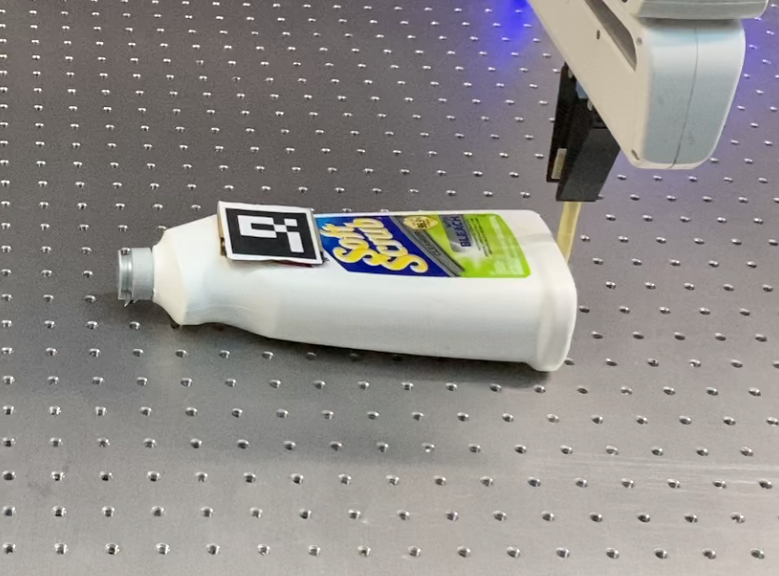} & \includegraphics[width=2.5cm]{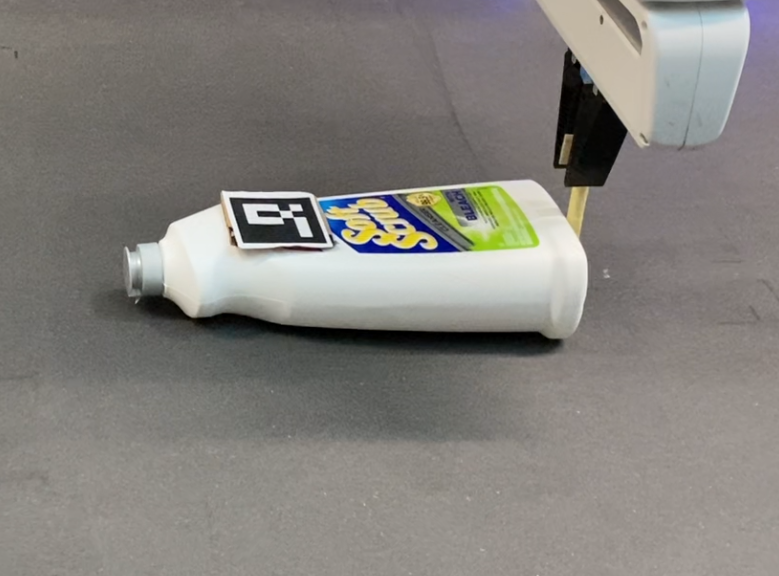} & \includegraphics[width=2.5cm]{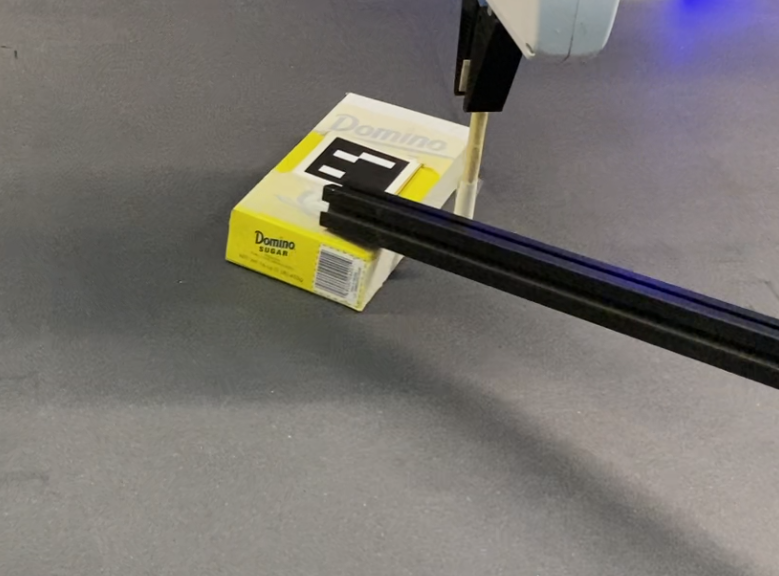} \\

\includegraphics[width=2.4cm]{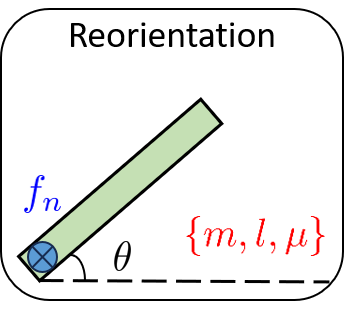} & \includegraphics[width=2.5cm]{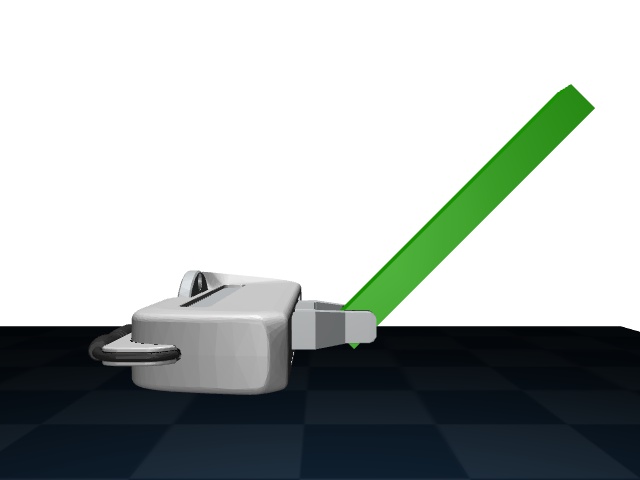} & \includegraphics[width=2.5cm]{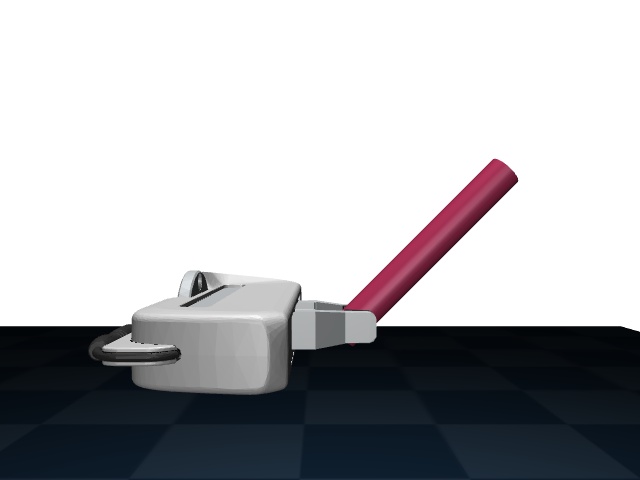} & \includegraphics[width=2.5cm]{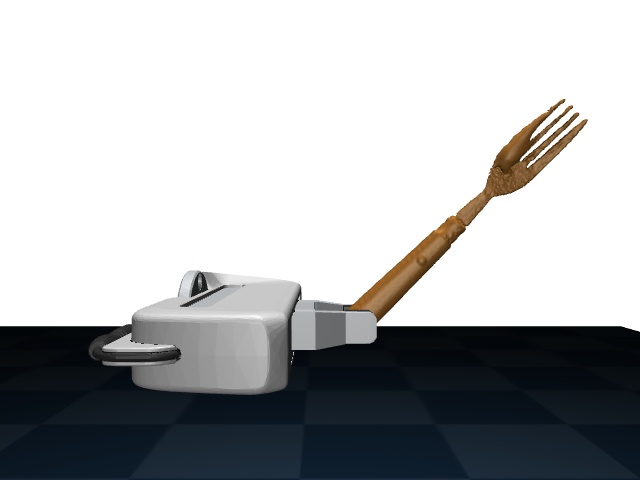} & \includegraphics[width=2.5cm]{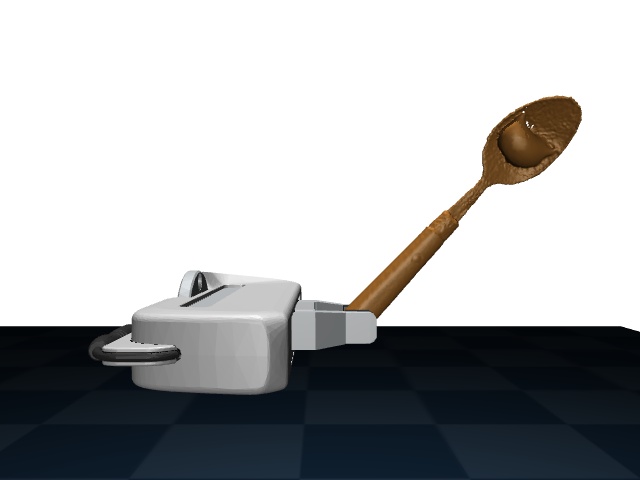} & \includegraphics[width=2.5cm]{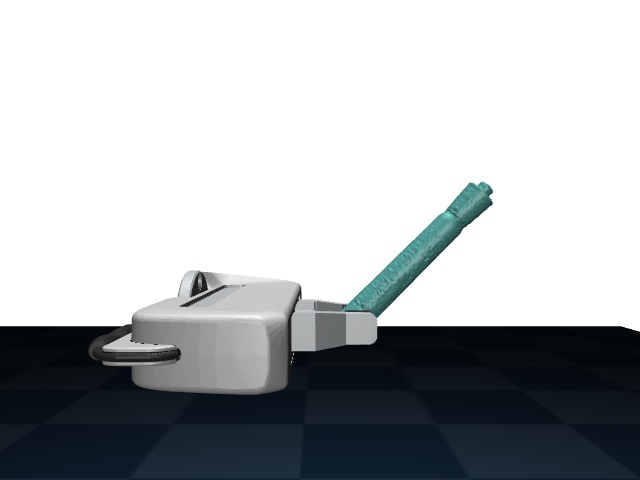} \\

\end{tabular}
\caption{\textbf{Deployment of the learned policy in a variety of contact-rich manipulation tasks.} The left images illustrate the state and action spaces for each primitive (\texttt{Hit}, \texttt{Push}, \texttt{Reorientation}) in black and blue, respectively, along with the parameter spaces represented by the red variables: $m$ for mass, $\mu$ for the friction coefficient, $r$ for the radius, and $l$ for the length. A single policy is trained for each primitive and deployed directly across a wide range of objects with varying shapes, weights, and friction parameters, while preserving instance-specific optimal behaviors.}
\label{fig:setup}
\end{figure*}

Many approaches have been proposed for robust planning and control through contacts, such as H-$\infty$ control \citep{franco2006robust}, Sliding Mode Control (SMC) \citep{shtessel2014sliding, edwards1998sliding} or Model Predictive Control (MPC) \citep{morari1999model, garcia1989model, lopez2019dynamic}. However, these methods often overlook instance-specific information and focus primarily on worst-case scenarios, resulting in conservative behaviors. To address this limitation, online system identification can be integrated to enable adaptive MPC \citep{adetola2009adaptive, adetola2011robust}. However, the diverse parameters and non-smooth contact dynamics result in computationally expensive parameter estimation and trajectory optimization, leading to poor performance for real-time control.

Learning manipulation policies that can quickly react to the physical world is thus required to advance further. Widely used approaches include reinforcement learning (RL) \citep{kaelbling1996reinforcement, kober2013reinforcement} and imitation learning (IL) \citep{Billard08chapter, schaal1999imitation}. Due to the high cost of data collection in real-world contact-rich scenarios, policy learning in simulation and transferring it to the real world has become a common strategy. To enable robust sim-to-real transfer, commonly adopted methods include domain randomization (DR) \citep{tobin2017domain, muratore2018domain} and domain adaptation (DA) \citep{bousmalis2018using, arndt2020meta, chebotar2019closing}. DR assumes that environment parameters follow specific distributions (e.g., normal or uniform), and the goal is to maximize (or minimize) the expected reward (or cost) under randomly sampled parameters during training. While this method provides good generalization to diverse instances, it often results in suboptimal and high-variance behaviors due to restrictive assumptions about the parameter distributions. In contrast, DA incorporates more instance-specific data during policy training, leading to optimal behaviors in the specific target domain but sacrificing generalization to other instances without additional fine-tuning.

Combining DR and DA to harness the benefits of both is a sensible strategy \citep{muratore2022neural, qi2023hand}, but how to effectively balance them is still an open question. It typically involves online system identification to capture domain-specific information and multi-domain learning to account for diverse instances. One promising approach to address this trade-off is to learn a base policy that uses diverse privileged environmental information as input and then relies on proprioceptive history to estimate the system parameters or a latent low-dimensional embedding online. For simplicity, we refer to both as the system parameters in this article, which are then used as input to the base policy to generate domain-specific optimal behaviors. This framework, referred to as Teacher-Student Policy in \citep{lee2020learning} and Rapid Motor Adaptation in \citep{kumar2021rma} and \citep{qi2023hand}, is more broadly termed \textit{explicit motor adaptation (EMA)} in this article, as the policies are typically represented as explicit feedforward functions that directly output actions given observations and estimated parameters. It can be viewed as a variant of DA that enhances generalization by including online parameter estimation. 

This approach has demonstrated impressive performance in quadrupedal locomotion across diverse terrains \citep{lee2020learning, kumar2021rma} and robust in-hand manipulation \citep{qi2023hand}. However, it typically relies on high-quality system identification \citep{qi2023hand} or additional training of a student policy to mimic the teacher policy \citep{lee2020learning}. These requirements pose significant challenges in contact-rich manipulation tasks, as they demand extensive data collection and considerable effort in model design and training, but remain difficult to achieve.

Instead, we propose \textit{implicit motor adaptation (IMA)} in this work. IMA can be seen as an elaborated version of DR, leveraging instance-aware distributions instead of being parameter-blind, while also retaining the advantage of DR in not relying heavily on good-quality system identification. It addresses sim-to-real transfer as a stochastic problem, pushing the limits for finding an optimal policy under sim-to-real uncertainty without requiring policy retraining for new instances. The policy is implicitly represented as $\argmax$ of the parameter-conditioned advantage function
\begin{equation}
    \bm{u}^* = \argmax_{\bm{u} \in \mathcal{U}} A({\bm{h}}, \bm{x}, \bm{u}) \quad \text{instead of} \quad \bm{u}^* = \pi_{\theta}({\bm{h}}, \bm{x}),
\end{equation}
where $\bm{h}$ is the proprioceptive history. EMA estimates the system parameter from $\bm{h}$ and directly inputs it into the base policy $\pi_{\theta}$ to output the action. In contrast, thanks to the implicit policy representation, IMA can retrieve the parameter-conditioned policy using a probabilistic estimate based on $\bm{h}$. This makes it particularly effective for contact-rich manipulation tasks, in which contact parameters identification is often challenging. Although implicit representations have been explored in imitation learning \citep{florence2022implicit} and as components of reinforcement learning \citep{haarnoja2017reinforcement, wang2022diffusion}, to the best of our knowledge, our work is the first to investigate this idea in the context of robust policy learning.

However, computing the parameter-conditioned advantage function and retrieving the policy can be computationally expensive, making it difficult to apply in real-world settings, particularly in contact-rich manipulation tasks where the robot has to continuously interact with its surroundings. To address this issue, we employ Tensor Train (TT) \citep{oseledets2011tensor} as function representation, enabling fast computation of the advantage function and efficient online policy retrieval.

This article builds on our previous work \citep{Xue24CORL}, where we introduced domain contraction, a mechanism for retrieving policies with a known probabilistic distribution. Here, we extend this concept by integrating domain contraction into a comprehensive sim-to-real transfer framework, namely \textit{implicit motor adaptation}. The framework incorporates an online system identification module that probabilistically estimates environmental parameters from proprioceptive history and retrieves the parameter-conditioned policy for real-world manipulation from a base policy trained in simulation.

In summary, beyond employing Tensor Train as an implicit policy representation and utilizing domain contraction for probabilistic policy retrieval, this paper introduces the following new contributions compared to \citep{Xue24CORL}:

1) A novel framework called \textit{implicit motor adaptation (IMA)} for robust contact-rich manipulation, highlighting the promise of implicit representation for robust policy learning compared with the widely used \textit{explicit motor adaptation (EMA)} framework.

2) A probabilistic system adaptation module allowing the robot to adapt its behavior in a probabilistic manner without relying on high-quality system identification or additional student policy training.

3) Theoretical proofs and numerical experiments demonstrating the effectiveness of IMA compared to EMA.

\section{Related Work}
\subsection{Learning for Contact-rich Manipulation}

 Many approaches have been proposed to learn manipulation policies for contact-rich manipulation, including Behavior Cloning (BC) \citep{florence2022implicit, torabi2018behavioral}, Deep Reinforcement Learning (DRL) \citep{pertsch2021accelerating, qi2023hand, lee2020learning}, and Approximate Dynamic Programming (ADP) \citep{powell2007approximate, werbos1992approximate}. In our work, we employ an ADP approach called Tensor Train Policy Iteration (TTPI) \citep{Shetty24ICLR} for policy learning, where the state space is augmented with system parameters to enable subsequent parameter-conditioned policy retrieval given different instances. However, our proposed approach is flexible and can be integrated with any policy learning technique, provided the policy can be implicitly represented by advantage functions or energy-based functions.

\subsection{Implicit Policy Representation} 
Instead of being expressed as an explicit function, the action policy can be described implicitly through function optimization. Implicit representations have been widely used in various approaches, including energy-based models (EBMs) \citep{lecun2006tutorial, du2019implicit}, diffusion models \citep{yang2023diffusion, ho2020denoising}, and tensor networks \citep{Orus2019tensor, stoudenmire2016supervised}. In imitation learning, behavior cloning can be formulated using EBMs \citep{florence2022implicit}, highlighting several advantages such as multimodal representation and long-horizon visuomotor policy learning. In methods based on dynamic programming, the control policy is naturally obtained by optimizing objective functions, making it more intuitive to represent the policy implicitly. For example, reactive policy learning can be reformulated as solving sequence-of-constraints model predictive control (MPC) \citep{toussaint2022sequence}, demonstrating strong generalization and improved compositionality using implicit functions. Similarly, implicit models can be utilized as policy representations in reinforcement learning, emphasizing their expressiveness in action generation \citep{haarnoja2017reinforcement,wang2022diffusion}. For high-dimensional action spaces, Tensor Train (TT) can be emplyed as a low-rank representation for state-value and advantage functions, enabling implicit retrieval of control actions \citep{Shetty24ICLR,tal2018continuous}. Our work is inspired by the performance of implicit representation in control policy learning and aims to explore its potential for robust policy learning,  which has not been investigated so far.

\subsection{Sim-to-real transfer} 
Obtaining large amounts of real-world data for primitive learning is challenging. Therefore, robot learning in simulation and transferring to the real world is a promising idea \citep{peng2018sim}. However, the reality gap between simulation and the real world poses a significant challenge to such idea. 
If the target domain is known and specific, Domain Adaptation (DA) \citep{bousmalis2018using, arndt2020meta, chebotar2019closing} can be an effective method, but its reliance on domain-specific data limits the generalization to other scenarios. Alternatively, Domain Randomization (DR) \citep{tobin2017domain} seeks to develop a robust policy by introducing random variations into the simulation parameters. While this method offers good generalization, it often results in suboptimal and high-variance behaviors due to the restrictive assumptions about environment parameter distribution (e.g., normal or uniform). 

Combining DA and DR can be a promising strategy to balance generalization and optimal performance. The basic idea is to adapt the parameter distribution by leveraging differences between simulated and reference environment \citep{mehta2020active, muratore2022neural, ramos2019bayessim}. However, policies developed through such methods are often tailored to the reference environment or target domain. Recently, explicit motor adaptation (EMA) \citep{yu2017preparing} has demonstrated remarkable success in learning robust locomotion \citep{lee2020learning, kumar2021rma} and manipulation primitives \citep{qi2023hand}. In this approach, a teacher policy is trained in simulation with various domain parameters, and a student policy learns to replicate this behavior, with a low-dimensional embedding of proprioceptive history as input. Our work closely follows this paradigm but introduces a more efficient way for learning the parameter-augmented teacher policy using tensor approximation. Additionally, our method eliminates the need for high-quality system identification and subsequent student policy training, as the parameter-conditioned robust policy can be directly derived from the teacher policy given a rough parameter distribution.

\subsection{Tensor Train for function approximation} 
A multidimensional function can be approximated by a tensor, where each element in the tensor is the value of the function given the discretized inputs. The continuous value of the function can then be obtained by interpolating among tensor elements. However, storing the full tensor for a high-dimensional function can be challenging. To address this issue, Tensor Train (TT) was proposed to approximate the tensor using several third-order tensor cores. The widely used methods include TT-SVD \citep{oseledets2011tensor} and TT-cross \citep{oseledets2010_ttcross1}. Furthermore, TTGO \citep{shetty2016tensor} was proposed for finding globally optimal solutions given functions in TT format. TTPI \citep{Shetty24ICLR} was then introduced to learn control policies through tensor approximation, showing superior performance on several hybrid control problems. Logic-Skill Programming (LSP) \citep{Xue24RSS} expands the operational space of TTPI by incorporating first-order logic to sequence policies. Building on TTGO and TTPI, our work extends these methods to learn robust policies, which further enhances the capabilities of LSP. We additionally demonstrate that the TT format is a suitable structure for efficient parameter-conditioned policy retrieval.

\section{Background}
\label{sec:background}
\subsection{Tensors as Discrete Analogue of a Function}

A multivariate function $F(x_1,\ldots, x_d)$ defined over a rectangular domain made up of the Cartesian product of intervals (or discrete sets) $I_1 \times \cdots \times I_d$ can be discretized by evaluating it at points in the set $\mc{X} = \{ (x^{i_1}_1,\ldots,x^{i_d}_d): x^{i_k}_k \in I_k, i_k \in \{1,\ldots, n_k\} \}$. This gives us a tensor $\bm{\mc{F}}$, a discrete version of $F$, where $\bm{\mc{F}}_{(i_1,\ldots, i_d)} = F(x^{i_1}_1,\ldots,x^{i_d}_d), \forall (i_1,\ldots, i_d)\in \mc{I}_{\mc{X}}$, and $\mc{I}_{\mc{X}} = \{ (i_1,\ldots,i_d): i_k \in \{1,\ldots, n_k\}, k \in \{1,\ldots, d\} \}$. The value of $F$ at any point in the domain can then be approximated by interpolating between the elements of the tensor $\bm{\mc{F}}$.

\subsection{Tensor Networks and Tensor Train Decomposition}

Naively approximating a high-dimensional function using a tensor is intractable due to the combinatorial and storage complexities of the tensor $(\mc{O}(n^d))$. Tensor networks mitigate the storage issue by decomposing the tensor into factors with fewer elements, akin to using Singular Value Decomposition (SVD) to represent a large matrix. In this paper, we explore the use of Tensor Train (TT), a type of Tensor Network that represents a high-dimensional tensor using several third-order tensors called \textit{cores}, as shown in Fig.~\ref{fig:tt_format}.
\begin{figure}[htbp]
	\centering
	\begin{minipage}{0.45\textwidth}
		\centering
		\includegraphics[width=\textwidth]{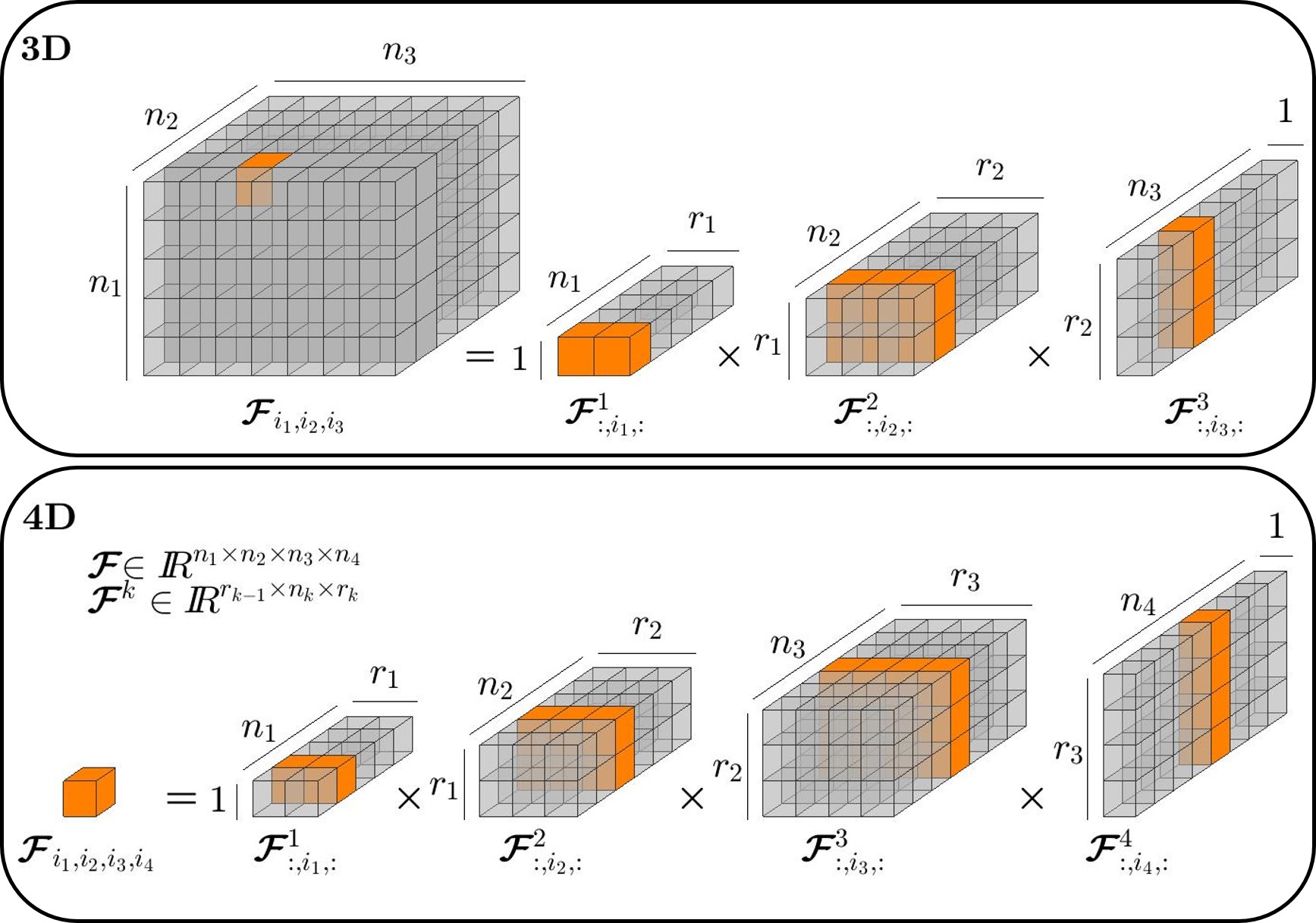} 
	\end{minipage}
	\hfill
		\caption{TT decomposition generalizes matrix decomposition techniques to higher-dimensional arrays. In TT format, an element in a tensor can be obtained by multiplying specific slices of the core tensors. The figure presents examples of third-order, and fourth-order tensors. Image adapted from \cite{shetty2016tensor}.}
		\label{fig:tt_format}
\end{figure}

We can access the element indexed $(i_{1},\ldots,i_{d})$ of the tensor in this format simply by multiplying matrix slices from the cores:
\begin{equation}
	\label{eq:tt_rep}
	\bm{\mc{F}}_{(i_{1},\ldots,i_{d})} = \bm{\mc{F}}^1_{:,i_1,:}\bm{\mc{F}}^2_{:,i_2,:}\cdot\cdot\cdot \bm{\mc{F}}^d_{:,i_d,:},
\end{equation}
where $\bm{\mc{F}}^k_{:,i_k,:} \in \mb{R}^{r_{k-1} \times r_k}$ represents the $i_{k}$-th frontal slice (a matrix) of the third-order tensor $\bm{\mc{F}}^k$. For any given tensor, there always exists a TT decomposition \citep{oseledets2011tensor}. This low-rank structure further facilitates sampling and optimization for robot planning and control.

There are several ways to acquire a TT model, including TT-SVD \citep{oseledets2011tensor} and TT-Cross \citep{oseledets2010_ttcross1,savostyanov2011_ttcross2}. TT-SVD extends the SVD decomposition from matrix level to a high-dimensional tensor level. However, it needs to store the full tensor first, which is impractical to very high-dimensional functions. TT-Cross solves this limitation by selectively evaluating function $F$ on a subset of elements, avoiding the need to store the entire tensor. 

\subsection{Function approximation using Tensor Train}
\label{sec:function_apprx}
Given the discrete analogue tensor $\bm{\mc{F}}$ of a function $F$, we obtain the continuous approximation by spline-based interpolation of the TT cores corresponding to the continuous variables. For example, we can use linear interpolation for the cores (i.e., between the matrix slices of the core) and define a matrix-valued function corresponding to each core $k \in \{1,\ldots, d\}$,
\begin{equation}
	\bm{F}^k(x_k) = \frac{x_k-x^{i_k}_k}{x^{i_k+1}_k-x^{i_k}_k}\bm{\mc{F}}^k_{:,i_k+1,:} +\frac{x^{i_k+1}_k-x_k}{x^{i_k+1}_k-x^{i_k}_k}\bm{\mc{F}}^k_{:,i_k,:},
    \label{eq: interp_tt}
\end{equation}
where $x^{i_k}_k\le x_k \le x^{i_k+1}_k$ and $\bm{F}^k: I_k \subset \mb{R} \rightarrow \mb{R}^{r_{k-1} \times r_k}$ with $r_0=r_d=1$. This induces a continuous approximation of $F$ given by
\begin{equation}
	\label{eq: continuous_tt}
	F(x_1,\ldots,x_d) \approx \bm{F}^1(x_1) \cdots \bm{F}^d(x_d).
\end{equation}

\noindent This allows us to selectively do the interpolation only for the cores corresponding to continuous variables, and hence we can represent functions in TT format whose variables could be a mix of continuous and discrete elements.

\subsection{Global Optimization using Tensor Train (TTGO)}


In optimization problems, the goal is to find the decision variables $\bm{x}$ that maximize the objective function $f(\bm{x})$. TTGO~\citep{shetty2016tensor} frames this problem as the maximization of an unnormalized probability density function (PDF) $F(\bm{x})$, which is derived from $f(\bm{x})$ through a monotonically non-increasing transformation. The TT-Cross algorithm is then used to compute the TT approximation of $F(\bm{x})$, denoted by $\bm{\mc{F}}$:
\begin{align}
    &F(x_1, \dots, x_d) 
    \approx \sum_{\gamma_1=1}^{r_1} \sum_{\gamma_2=1}^{r_2} \cdots \sum_{\gamma_{d-1}=1}^{r_{d-1}} \notag \\
    &\quad \bm{\mc{F}}^1(1,\, x_1,\,\gamma_1)\bm{\mc{F}}^2(\gamma_1,\, x_2,\,\gamma_2) \cdots \bm{\mc{F}}^d(\gamma_{d-1},\, x_d,\,1),
\end{align}
where each $\bm{\mc{F}}^k$ is a TT core of size \(r_{k-1} \times n_k \times r_k\), with \(n_k\) denoting the number of discretization points for \(x_k\), and the TT ranks \(r_k\) typically remaining small even in high dimensions.

Given the TT representation of $F(\bm{x})$ as $\bm{\mc{F}}$, we can exploit its separable structure to perform \emph{coordinate-wise} optimization by iteratively selecting and refining promising candidates across dimensions. In contrast to naive methods that require exhaustive search over an exponentially large grid (i.e., \( \mathcal{O}(n^d) \)), the TT format enables localized and efficient search for a global (or near-global) optimum with significantly lower complexity (i.e., \( \mathcal{O}(ndr^2) \)). Furthermore, this approach is entirely gradient-free, making it well-suited for high-dimensional and non-convex landscapes. For further technical details on these procedures, we refer the readers to~\citep{oseledets2011tensor, sozykin2022ttopt, shetty2016tensor}.

\subsection{Generalized Policy Iteration using Tensor Train (TTPI)}
Optimal control of dynamical systems with nonlinear dynamics poses a substantial challenge in robotics. To address this, Generalized Policy Iteration using Tensor Train (TTPI) was proposed \citep{Shetty24ICLR}, leveraging tensor approximation and approximate dynamic programming. This method approximates state-value and advantage functions using Tensor Train (TT), effectively mitigating the curse of dimensionality. The low-rank structure of TT enables the use of TTGO to find near-global solutions during policy retrieval from the advantage function, even under complex nonlinear system constraints, surpassing the capabilities of existing neural network-based algorithms. TTPI has demonstrated superior performance on several hybrid control problems. In this article, we further extend this approach to learning robust manipulation primitives for contact-rich tasks involving diverse contact parameters.

\section{Problem formulation}
Let us consider a discrete-time dynamical system:
\begin{itemize}
    \item \textbf{State} $\bm{x} \in  \Omega_{\bm{x}} \subseteq \mathbb{R}^m$: The system's state space.
    \item \textbf{Action} $\bm{u} \in  \Omega_{\bm{u}}\subseteq \mathbb{R}^n$: The system's action space.
    \item \textbf{Parameters} $\bm{\alpha} \in \Omega_{\bm{\alpha}}\subseteq \mathbb{R}^d$: The unknown or uncertain physical parameters of the system, such as mass, friction coefficient, etc. 
    \item \textbf{Dynamics Model} $f: \Omega_{\bm{\alpha}} \times \Omega_{\bm{x}} \times \Omega_{\bm{u}} \to \Omega_{\bm{x}}$:  The system's next state depends on current state $\bm{x}_t$, action $\bm{u}_t$ and the system parameters $\bm{\alpha}$. 
    \item \textbf{Reward function} $R: \Omega_{\bm{x}}\times \Omega_{\bm{u}} \to \mathbb{R}$: The immediate reward, which depends on the current state $\bm{x}_t$, action $\bm{u}_t$.
    \item \textbf{Discount factor} $\gamma \in [0,1]$.

\end{itemize}
This system can be described by a Markov Decision Process (MDP) $\mathcal{M} = \{\Omega_{\bm{x}}, \Omega_{\bm{u}}, \Omega_{\bm{\alpha}}, f, R, \gamma \}$. The objective is to find a policy $\pi$ that maximizes the expected cumulative reward
\begin{equation}
 \begin{aligned}
\mathbb{E}_{\pi} [ \sum_{t=0}^{\infty} \gamma^t R(\bm{x}_{t}, \pi(\bm{\alpha}, \bm{x}_t)) \mid \bm{x}_{0}\!=\! \bm{x} ].
    \label{eq: determin_MDP}
\end{aligned}
\end{equation}
However, knowing the exact value of $\bm{\alpha}$ in the real world is intractable. This information can, however, be estimated using multimodal sensors \citep{lee2020making} or proprioceptive history \citep{bianchini2023simultaneous, lee2020learning}. Given the sensor noise and model inaccuracy, it is better to estimate the system parameters as a probabilistic distribution, denoted as $\hat{\bm{\alpha}} \sim P(\hat{\bm{\alpha}})$. Therefore, our approach aims to find a policy $\pi$ that maximizes 
\begin{equation}
 \begin{aligned}
\mathbb{E}_{\hat{\bm{\alpha}} \sim P(\hat{\bm{\alpha}})} \left[   \mathbb{E}_{\pi}  \left( \sum_{t=0}^{\infty} \gamma^t R(\bm{x}_{t}, \pi(\hat{\bm{\alpha}}, \bm{x}_t)) \mid \bm{x}_{0}\!=\! \bm{x}  \right) \right].
    \label{eq: MDP}
\end{aligned}
\end{equation}

Solving \eqref{eq: MDP} involves knowing the probability distribution $P(\hat{\bm{\alpha}})$ and then finding the policy $\pi$ that maximizes the expected cumulative reward over this distribution. One approach to address this problem is to learn control policies separately for each instance with its specific parameter distribution, akin to performing domain adaptation within each target domain. However, a single manipulation primitive typically involves numerous instances, making this approach too time-consuming and impractical. An alternative is to learn an augmented base policy that encompasses multiple instances and retrieve instance-specific policies using estimated instance information, similarly to the \textit{explicit motor adaptation (EMA)} strategy. However, EMA cannot solve \eqref{eq: MDP} as it focuses only on one specific parameter instance rather than a distribution.

In this work, we present an \textit{implicit motor adaptation (IMA)} approach that addresses the problem of probabilistic system identification and, more importantly, the retrieval of a robust policy conditioned on the distribution of parameters. Additionally, to reduce the computational burden of online control, we employ tensor factorization as an implicit representation of the base policy, enabling efficient parameter-conditioned policy retrieval through tensor core products.

\section{Implicit Motor Adaptation (IMA)}
\label{sec:method}

In this section, we provide a detailed introduction to \textit{implicit motor adaptation (IMA)}, focusing on its three key components: 1) parameter-augmented base policy learning, 2) probabilistic system adaptation conditioned on proprioceptive history, and 3) parameter-conditioned policy retrieval via domain contraction. The full pipeline is shown in Fig.~\ref{fig:pipeline}.

\begin{figure*}[t]
	\centering
	\includegraphics[width=0.9\textwidth]{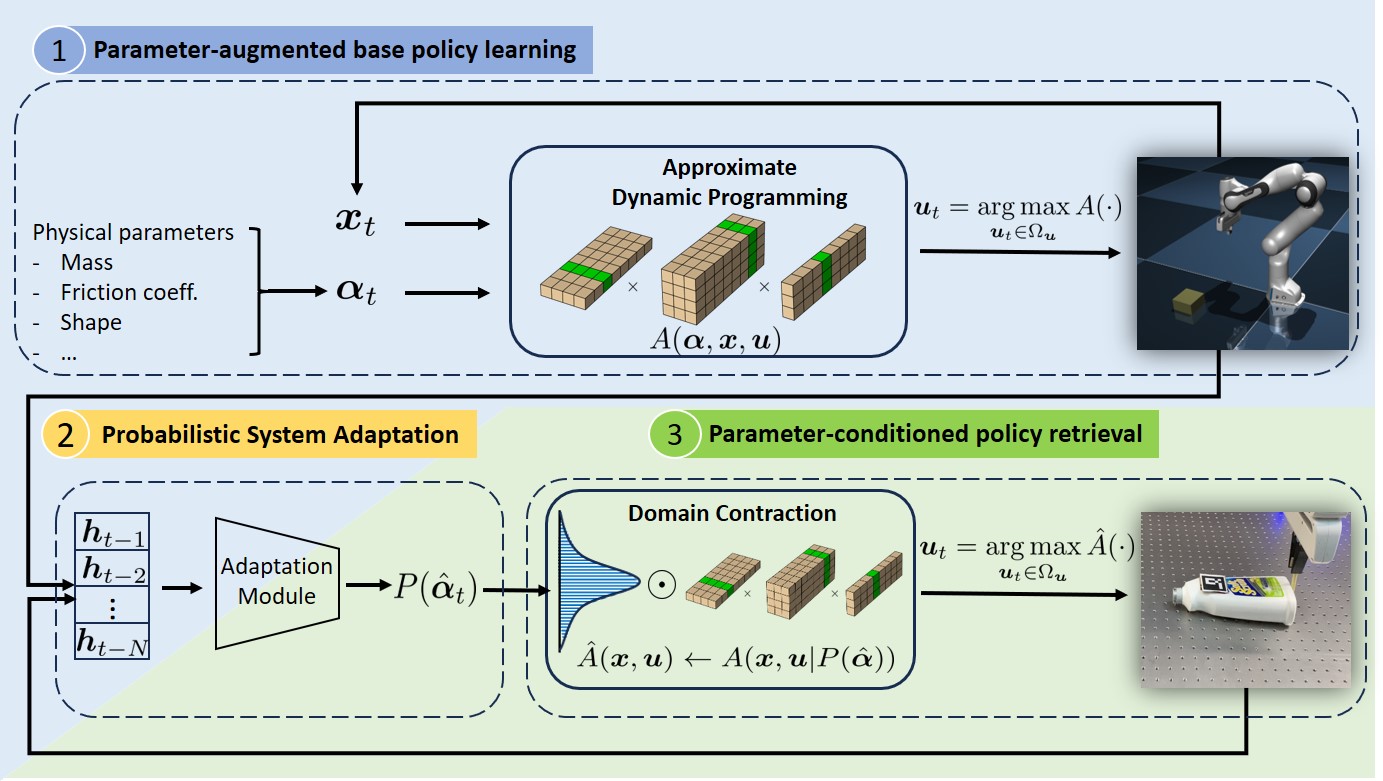}
	\caption{Pipeline of the proposed approach, including (1) parameter-augmented base policy learning, (2) probabilistic system adaptation with proprioceptive history, and (3) parameter-conditioned policy retrieval. The base policy and parameter-conditioned policy are implicitly represented by the corresponding advantage functions $A(\bm{\alpha}, \bm{x}, \bm{u})$ and $\hat{A}(\bm{x}, \bm{u})$, respectively. Blue-shaded modules are trained in simulation, and green-shaded ones are used in deployment, with the probabilistic system adaptation module bridging both stages.}
	\label{fig:pipeline}
\end{figure*} 

\subsection{Parameter-augmented Base Policy Learning}
\label{sec: param_aug_pl}

Similarly to the multi-goal setting in DRL \citep{liu2022goal}, we first train a parameter-augmented policy by augmenting the state space with parameters, treating both the state $\bm{x}$ and parameter $\bm{\alpha}$ as inputs and the action $\bm{u}$ as the output. This policy serves as the base policy for parameter-conditioned policy retrieval in Section \ref{sec:param_con_retrieval}, and is also used for data collection to train the probabilistic adaptation module discussed in Section \ref{sec: sys_id}. The resulting parameter-augmented bellman equation is
\begin{equation}
 \begin{aligned}
      V(\bm{\alpha}, \bm{x}) &=  \mathbb{E}_{\pi}  \biggl[ \sum_{t=0}^{\infty} \gamma^t R(\bm{x}_{t}, \bm{u}_{t}) \mid \bm{x}_{0}\!=\! \bm{x}  \biggr],\\
      A(\bm{\alpha}, \bm{x}, \bm{u}) &= R(\bm{x}, \bm{u}) + \gamma (V(f(\bm{\alpha}, \bm{x}, \bm{u})) - V(\bm{\alpha}, \bm{x})), \\
      \pi(\bm{\alpha}, \bm{x}) &= \arg\max_{\bm{u} \in \Omega_{\bm{u}} } A(\bm{\alpha}, \bm{x}, \bm{u}).
    \label{eq: param_augment VA}
\end{aligned}
\end{equation}
where $V(\bm{\alpha}, \bm{x})$ is the parameter-augmented state-value function, and $A(\bm{\alpha}, \bm{x}, \bm{u})$ is the parameter-augmented advantage function. 

Solving \eqref{eq: param_augment VA} is typically a challenging ADP problem due to the curse of dimensionality. In this article, we tackle this challenge by building on our prior work, TTPI \citep{Shetty24ICLR}, which utilizes tensor factorization techniques for function approximation. The Value Iteration algorithm \citep{pashenkova1996value} is used to determine the optimal value function. At any iteration $k$, the $(k+1)$-th value function approximation is computed as
\begin{equation}
\label{eq:tt_value_iteration}
    \begin{aligned}
         V^{k+1} =& \text{TT-Cross}(\mc{B}^{\pi_k}V^k, \epsilon), \\
          \mc{B}^{\pi_k}V^k(\bm{\alpha}, \bm{x}) =& R(\bm{x},\pi^k(\bm{\alpha}, \bm{x})) + \gamma V^k(f(\bm{\alpha},\pi^k(\bm{\alpha}, \bm{x})))\\
         \pi^{k}(\bm{x}) =& \underset{\bm{u} \in \Omega_{\bm{u}}}{\mathrm{argmax}} A^k(\bm{\alpha}, \bm{x},\bm{u}), \\
         A^k(\bm{\alpha}, \bm{x}, \bm{u}) =& R(\bm{x}, \bm{u}) + \gamma (V^{k}(f(\bm{\alpha}, \bm{x}, \bm{u})) -V^k(\bm{\alpha}, \bm{x})), \\
    \end{aligned}
\end{equation}
where $\epsilon$ is the accuracy threshold of TT-cross approximation.

To compute $V^{k+1}$ in TT format, the function $\mc{B}^{\pi_k}V^k$ is queried iteratively using $\text{TT-Cross}(\mc{B}^{\pi_k}V^k, \epsilon)$, with batches of states (usually ranging from 1000 to 100,000 in practice). This requires computing the policy $\pi^k(\bm{\alpha}, \bm{x}) = \underset{\bm{u} \in \Omega_{\bm{u}}}{\argmax}\, A^{k}(\bm{\alpha}, \bm{x},\bm{u})$ numerous times across several iterations, making a fast computation of $ \underset{\bm{u}  \in \Omega_{\bm{u}}}{\mathrm{argmax}} A^k(\bm{\alpha}, \bm{x}, \bm{u})$ in batch form crucial.

To resolve the bottleneck, the advantage function $A^k$ is computed in TT format using TT-Cross. This is efficient as the calculation only requires evaluating $V^k$ and $R$, which are cheap to compute. This enables the use of TTGO \citep{shetty2016tensor}, an efficient optimization technique for a function in TT format. As a result, solutions for $\pi^k(\bm{\alpha}, \bm{x}) = \underset{\bm{u}  \in \Omega_{\bm{u}} }{\mathrm{argmax}}\, A^k(\bm{\alpha}, \bm{x}, \bm{u})$ over batches of states can be obtained quickly, leading to efficient policy retrieval and value iteration. The optimal value function and advantage function are obtained upon algorithm convergence, resulting in functions in TT format:
\begin{equation}
\begin{aligned}
	\label{eq:tt_V}
	V(\bm{\alpha}, \bm{x}) & \approx \bm{\mc{V}}(\alpha_{1:d}, \bm{x}_{1:m})\\ &= \bm{\mc{V}}^1_{:,i_{1},:} \cdots \bm{\mc{V}}^{d}_{:,i_{d},:} \; \bm{\mc{V}}^{d+1}_{:,i_{d+1},:}  \cdots \bm{\mc{V}}^{d+m}_{:,i_{d+m},:},
 \end{aligned}
\end{equation}
\begin{equation}
	\label{eq:tt_A}
        \begin{aligned}
	A(\bm{\alpha}, \bm{x},\bm{u}) &\approx \bm{\mc{A}}(\alpha_{1:d}, \bm{x}_{1:m}, \bm{u}_{1:l}) \\&= \bm{\mc{A}}^1_{:,i_{1},:} \cdots \bm{\mc{A}}^{d}_{:,i_{d},:} \; \bm{\mc{A}}^{d+1}_{:,i_{d+1},:}  \cdots \bm{\mc{A}}^{d+m}_{:,i_{d+m},:} \\ & \quad  \bm{\mc{A}}^{d+m+1}_{:,i_{d+m+1},:} \cdots  \bm{\mc{A}}^{d+m+l}_{:,i_{d+m+l},:}.
         \end{aligned}
\end{equation}

\subsection{Probabilistic System Adaptation}
\label{sec: sys_id}

The system parameters $\bm{\alpha}$ (also called privileged environment information) is typically not accessible during execution in the real world. However, we can probabilistically estimate it as $\hat{\bm{\alpha}}$, drawn from a probability distribution $P(\hat{\bm{\alpha}})$. This distribution can be inferred from the proprioceptive history, namely
\begin{equation}
    P(\hat{\bm{\alpha}}_t) = \phi (\bm{x}_{t-k:t-1}, \bm{u}_{t-k, t-1}).
    \label{eq:prob_adap}
\end{equation}

We name this process as probabilistic system adaptation. To learn the mapping between the proprioceptive history and $P(\hat{\bm{\alpha}})$, a dataset is required containing the state-action history $\bm{h} = \{ \bm{x}_{t-k:t-1}, \bm{u}_{t-k:t-1} \}$ and the corresponding system parameter $\bm{\alpha}_t$, which is accessible in simulation. Various approaches, such as energy-based models (EBMs) \cite{lecun2006tutorial} and diffusion models \cite{ho2020denoising}, can be employed to learn this distribution. In this article, the subsequent probabilistic policy retrieval (Section \ref{sec:param_con_retrieval}) reduces the strong reliance on precise parameter estimation. We use a straightforward multilayer perceptron (MLP) to approximate $\phi$, while noting that more sophisticated models could be explored in future extensions.

In practice, the loss function is defined as MSE $(\bm{\nu}_t, \bm{\alpha}_t) = \|\bm{\nu}_t - \bm{\alpha}_t\|^2$, where $\bm{\nu}_t$ is the output of MLP. We assume that $P(\hat{\bm{\alpha}}_t)$ follows a uniform distribution $\mathbb{U}(\bm{\nu}_t - \bm{w}/2, \bm{\nu}_t + \bm{w}/2)$, where $\bm{w}$ is a hyperparameter representing the bandwidth of $\mathbb{U}$. Note that our framework is flexible and can incorporate other distributions as well.

\subsection{Parameter-conditioned policy retrieval}
\label{sec:param_con_retrieval}
\begin{figure}[t]
	\centering
	\includegraphics[width=0.5\textwidth]{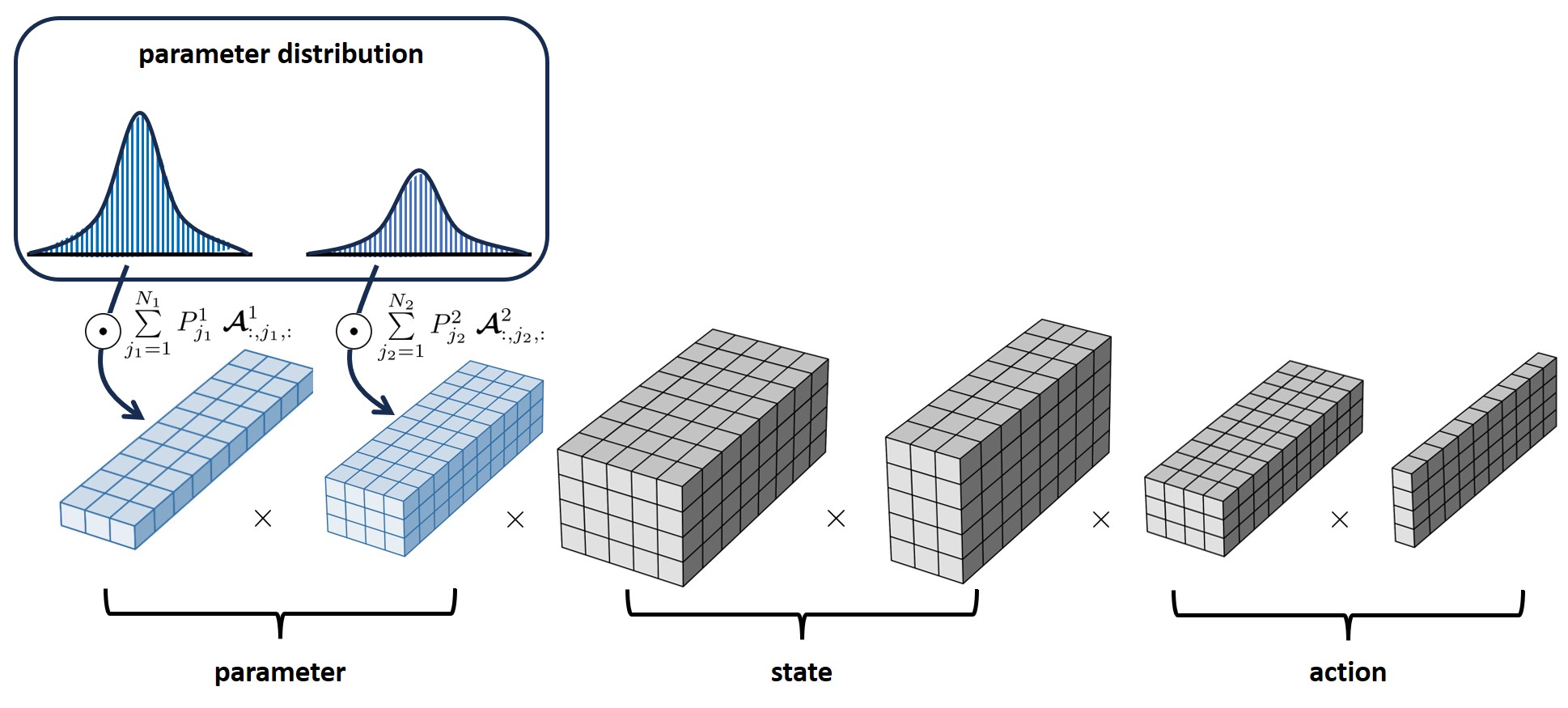}
	\caption{Domain contraction in TT format. \\The parameter-augmented advantage function in TT format typically includes separate 3rd-order cores for different dimensionality, such as parameter, state and action. Given a probabilistic parameter distribution, we can retrieve the parameter-conditioned policy by making product of parameter distributions and corresponding TT cores.}
	\label{fig:TT_dc}
\end{figure} 

\begin{figure}[t]
	\centering
	\includegraphics[width=0.42\textwidth]{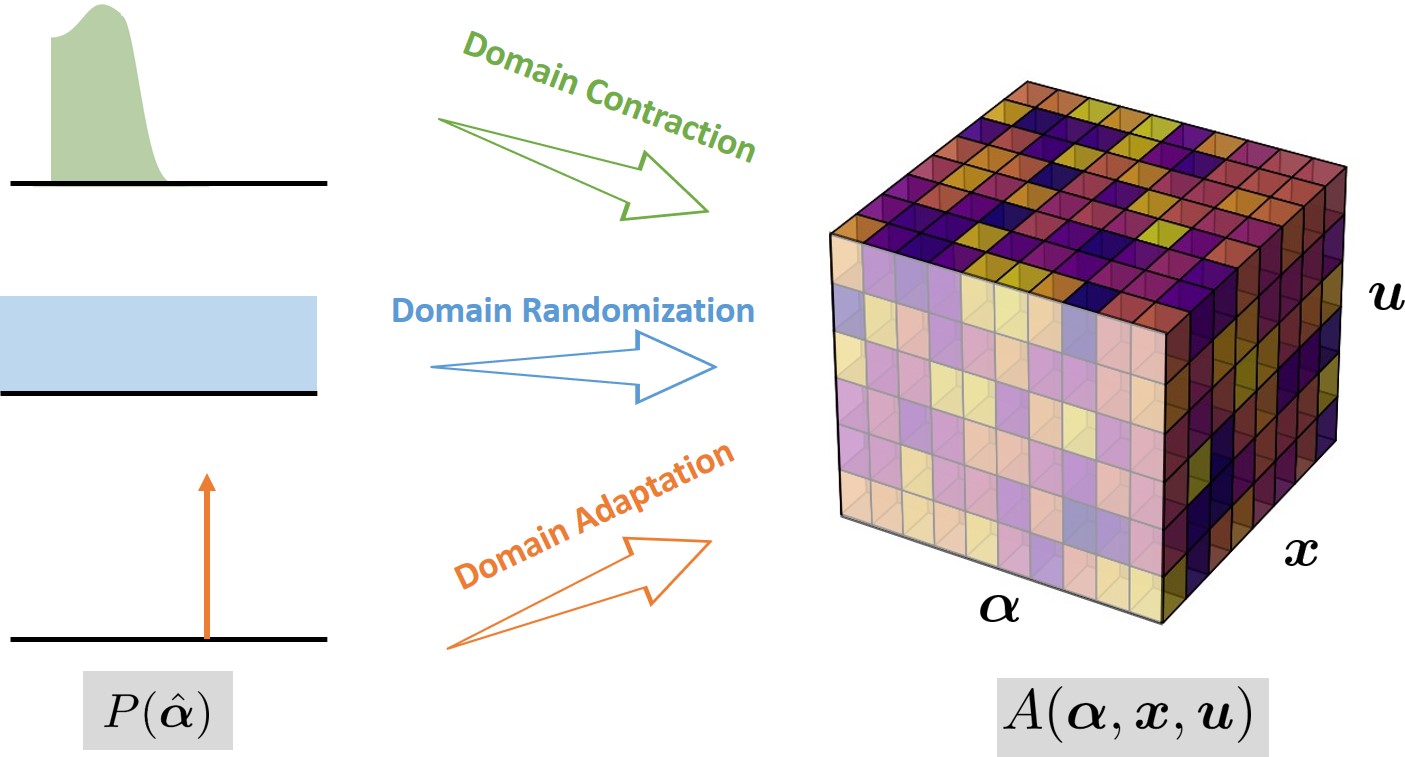}
	\caption{Domain contraction unifies domain randomization and domain adaptation by giving different parameter distributions.}
	\label{fig:dc_dr_da}
\end{figure} 

Without loss of generality, we assume the contact-rich manipulation task involves $d$ types of different parameters, such as mass, size and friction coefficient. The parameter space $\Omega_{{\alpha}}$ is therefore composed of $d$ subspaces, namely $\Omega_{{\alpha}} = \Omega_{{\alpha}_1} \times \cdots \times \Omega_{{\alpha}_d}$. We assume each subspace is discretized by $N_i$ points. Let $\bm{\alpha}_j  = (\alpha_{j_1}, \cdots, \alpha_{j_d})$ represents one instance of domain parameter, at the discretization index $j$ across all $d$ dimensions. Similarly, the parameter distribution $P(\hat{\bm{\alpha}})$ is defined within $\Omega_P$, also composed of $d$ subspaces, namely $\Omega_{{P}} = \Omega_{P_1} \times \cdots \times \Omega_{P_d}$, and each subspace is discretized by $N_i$ points. Let ${P}_j = P_1(\hat{\alpha}_{j_1})  P_2(\hat{\alpha}_{j_2}) \cdots P_d(\hat{\alpha}_{j_d})$ represents the probability of the estimated value $\hat{\bm{\alpha}}_j$, where $P_i$ is the probability distribution at dimension $i$, $i \in \{1, \cdots, d\}$.

Given the estimated parameter distribution $P(\hat{\bm{\alpha}})$, obtaining the corresponding parameter-conditioned policy $\pi(\bm{x} |P(\hat{\bm{\alpha}}))$ is not free. It is neither simply averaging the parameter estimates $\sum_{j=1}^N P_j \hat{\bm{\alpha}}_j$ and directly feeding it into $\pi(\bm{\alpha}, \bm{x})$, nor merely taking the weighted sum of parameter-specific policies $\sum_{j=1}^N P_j \pi (\hat{\bm{\alpha}}_j, \bm{x})$. Instead, it involves finding the $\argmax$ of the weighted sum of parameter-specific advantage functions, namely
\begin{equation}
        \label{eq:full_adv}
	\begin{aligned}
		A(\bm{x},\bm{u} | P(\hat{\bm{\alpha}})) = \sum_{j_1=1}^{N_1} \cdots \sum_{j_d=1}^{N_d}  P_j  A_{\hat{\bm{\alpha}}_j} (\bm{x}, \bm{u}).
	\end{aligned}
\end{equation}

\noindent We call this process as domain contraction in our previous work \citep{Xue24CORL}, as illustrated in Fig.~\ref{fig:TT_dc}. Intuitively, domain contraction randomizes over a contracted domain based on instance-specific parameter distribution, instead of randomizing over the full domain by blindly ignoring all instance-specific information (as in domain randomization).
We further provide the theoretical proof below:

\begin{theorem}
Given the estimated parameter distribution $P(\hat{\bm{\alpha}})$, the parameter-conditioned policy can be retrieved from the weighted sum of parameter-specific advantage functions.
\end{theorem}

\begin{proof}
\label{proof: proof1}
Obtaining the parameter-conditioned policy involves optimizing the parameter-augmented function
\begin{equation}
	\label{eq:full_para_advantage}
	\begin{aligned}
		&A(\bm{x},\bm{u} | P(\hat{\bm{\alpha}})) 
        \\
        = &R(\bm{x},\bm{u}) + \gamma \Big(V \big(f( \bm{x},\bm{u} | P(\hat{\bm{\alpha}})) \big) - V (\bm{x}|P(\hat{\bm{\alpha}}))\Big), \\ 
        \forall &{(P, \bm{x},\bm{u})} \in \Omega_{P} \times \Omega_{\bm{x}} \times \Omega_{\bm{u}},\\
	\end{aligned}
\end{equation}
with
\begin{equation}
\begin{aligned}
 \centering
        &  V (\bm{x}|P(\hat{\bm{\alpha}}))         
         = \sum_{j_1=1}^{N_1} \cdots \sum_{j_d=1}^{N_d}  P_{(j_1, \cdots, j_d)} \; V (\bm{x}|\hat{\alpha}_{j_1}, \cdots, \hat{\alpha}_{j_d}), \\
        &\sum_{j_1=1}^{N_1} \cdots \sum_{j_d=1}^{N_d} P_{(j_1, \cdots, j_d)} = 1.
	\end{aligned}
        \label{eq: prob_v}
\end{equation}
For simplicity, we use $\hat{\bm{\alpha}}_j$ to represent $(\hat{\alpha}_{j_1}, \cdots, \hat{\alpha}_{j_d})$ and $P_j$ to represent $P_(j_1, \cdots, j_d)$.
Given that each subspace $\Omega_{{\alpha}_i}$ of the parameter space $\Omega_{{\alpha}}$ is discretized by $N_i$ points, the parameter-conditioned advantage function can be written as 
\begin{equation}
	\label{eq:long_full_adv}
	\begin{aligned}
		&A(\bm{x},\bm{u} | P(\hat{\bm{\alpha}})) \\
        = &R(\bm{x},\bm{u}) + \sum_{j_1=1}^{N_1} \cdots \sum_{j_d=1}^{N_d} P_j \; \gamma \big( V \big(f(\bm{x},\bm{u}|\hat{\bm{\alpha}}_{j})\big) - V (\bm{x}|\hat{\bm{\alpha}}_{j})\big).
	\end{aligned}
\end{equation}

The right side of \eqref{eq:long_full_adv} can be further rewritten as

\begin{equation}
\label{eq:long_full_adv_2}
    \sum_{j_1=1}^{N_1} \cdots \sum_{j_d=1}^{N_d}  P_j \Big( R(\bm{x},\bm{u}) + \gamma  \big( V \big(f(\bm{x},\bm{u}|\hat{\bm{\alpha}}_{j}) \big) - V (\bm{x}|\hat{\bm{\alpha}}_{j}) \big) \Big),
\end{equation}
and the parameter-specific advantage function is defined as
\begin{equation}
	\label{eq:specific_adv}
	\begin{aligned}
		A_{\hat{\bm{\alpha}}_j} (\bm{x}, \bm{u}) = R(\bm{x},\bm{u}) +  \gamma \big ( V \big(f(\bm{x},\bm{u}|\hat{\bm{\alpha}}_j)\big) - V (\bm{x}|\hat{\bm{\alpha}}_j) \big ).
	\end{aligned}
\end{equation}

From \eqref{eq:long_full_adv}, \eqref{eq:long_full_adv_2}, and \eqref{eq:specific_adv}, we can derive that the parameter-conditioned advantage function is the weighted sum of parameter-specific advantage functions, as shown in \eqref{eq:full_adv},
and the parameter-conditioned primitive policy can then be computed by
\begin{equation}
	\label{eq:robu_policy}
	\pi(\bm{x} |P(\hat{\bm{\alpha}})) = \arg\max_{\bm{u} \in \Omega_{\bm{u}}} A(\bm{x}, \bm{u}|P(\hat{\bm{\alpha}})).
\end{equation}

Note that the parameter-conditioned policy cannot be directly computed as the weighted sum of parameter-specific policies, since the $\arg\max$ operation does not have an associative property with respect to addition. 
\end{proof}

Computing \eqref{eq:full_adv} and then retrieve the parameter-conditioned policy can be computationally expensive due to the combinatorial complexity and $\argmax$ over an arbitray function. We address this issue by leveraging the separable structure of TT format for efficient algebraic operation, and its advantage of finding optimal solutions for functions in TT format. In Sec. \ref{sec: param_aug_pl}, we have approximated the parameter-augmented advantage functions in TT format. We define the tensor cores related to $\bm{x}$ and $\bm{u}$ in \eqref{eq:tt_A} as 
\begin{equation}
	\label{eq:tt_A_xu}
	\bm{\mc{A}}(\bm{x}_{1:m}, \bm{u}_{1:l}) = \bm{\mc{A}}^{d+1}_{:,i_{d+1},:}  \cdots \bm{\mc{A}}^{d+m}_{:,i_{d+m},:} \cdots  \bm{\mc{A}}^{d+m+l}_{:,i_{d+m+l},:}, 
\end{equation}
thus the parameter-specific advantage functions can be extracted as 
\begin{equation}
	\label{eq:tt_A_alpha}
        \begin{aligned}
	A_{\hat{\bm{\alpha}}_j} (\bm{x}, \bm{u}) &\approx \bm{\mc{A}}(\bm{x}_{1:m}, \bm{u}_{1:l}|\hat{\bm{\alpha}}_{j}) \\ &= \bm{\mc{A}}^1_{:,j_{1},:} \cdots \bm{\mc{A}}^{d}_{:,j_{d},:} \; \bm{\mc{A}}(\bm{x}_{1:m}, \bm{u}_{1:l}), 
        \end{aligned}
\end{equation}

The parameter probability $P_j$ for parameter instance $\hat{\bm{\alpha}}_j$ is given by
\begin{equation}
    P_{(j_{1},\ldots,j_{d})} = P_1(\hat{\alpha}_{j_1})  P_2(\hat{\alpha}_{j_2}) \cdots P_d(\hat{\alpha}_{j_d}).
\end{equation}

For simplicity, we denote $P_d(\hat{\alpha}_{j_d})$ as $P^d_{j_d}$. The parameter-conditioned advantage function in \eqref{eq:full_adv} can then be efficiently computed as the weighted sum of parameter-specific advantage functions for each dimension, utilizing the separable structure of the TT model, namely

\begin{equation}
	\begin{aligned}
		& A(\bm{x},\bm{u}|P(\hat{\bm{\alpha}})) \approx  \sum_{j_1=1}^{N_1} \cdots \sum_{j_d=1}^{N_d} P_{(j_{1},\ldots,j_{d})} \bm{\mc{A}}(\bm{x}_{1:m}, \bm{u}_{1:l}|\hat{\bm{\alpha}}_{j})\\
			&= \sum_{j_1=1}^{N_1}  P^1_{j_1} \; \bm{\mc{A}}^{1}_{:,j_{1},:} \cdots \sum_{j_d=1}^{N_d} P^d_{j_d} \; \bm{\mc{A}}^{d}_{:,j_{d},:} \; \bm{\mc{A}}(\bm{x}_{1:m}, \bm{u}_{1:l}).
	\end{aligned}
\end{equation}

The computation of weighted sum is in tensor core level, which is much more efficient than in the function level. After obtaining the parameter-conditioned advantage function, we can retrieve the parameter-conditioned policy using TTGO \citep{shetty2016tensor}, a specialized method for efficiently finding globally optimal solutions given functions in TT format.

To retrieve such parameter-conditioned policies, knowing the parameter distribution $P(\hat{\bm{\alpha}})$ is crucial. Domain randomization (DR) and domain adaptation (DA) rely on assumed distributions during training, while domain contraction (DC) offers a better way to leverage domain knowledge during execution, enabling optimal behaviors while preserving generalization capability. We further demonstrate that DR and DA are two special cases of DC, as shown in Figure \ref{fig:dc_dr_da} and detailed in \citep{Xue24CORL}.

\begin{figure*}[htbp]
    \centering

    \vspace{1em} 

    \begin{minipage}{0.325\linewidth}
        \subfloat[\centering Ground truth]{\includegraphics[width=0.95\linewidth]{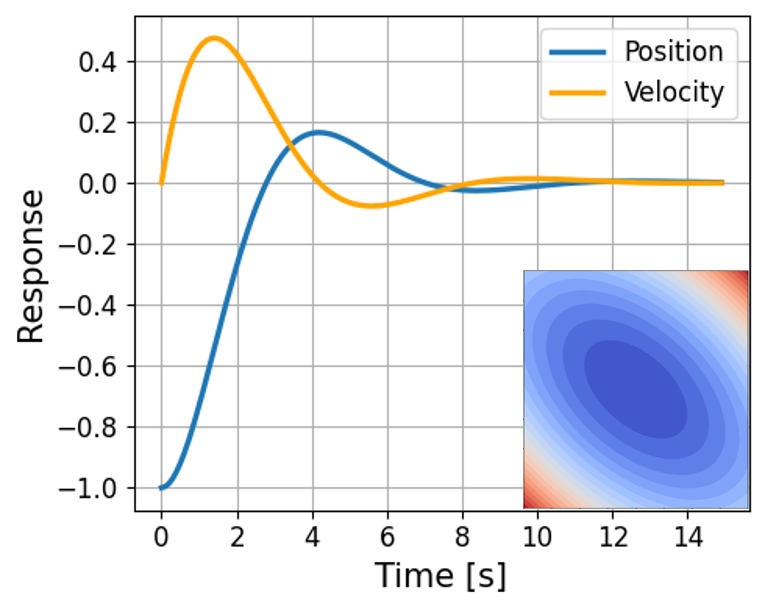}\label{fig:real_v}}
    \end{minipage}
    \hfill
    \begin{minipage}{0.325\linewidth}
        \subfloat[\centering EMA result]{\includegraphics[width=0.95\linewidth]{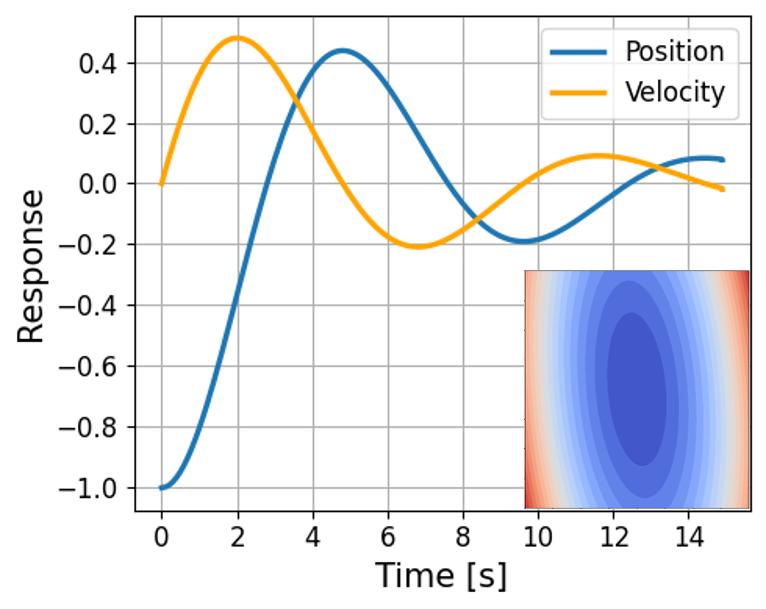}\label{fig:ema_v}}
    \end{minipage}
    \hfill
    \begin{minipage}{0.325\linewidth}
        \subfloat[\centering IMA result]{\includegraphics[width=0.95\linewidth]{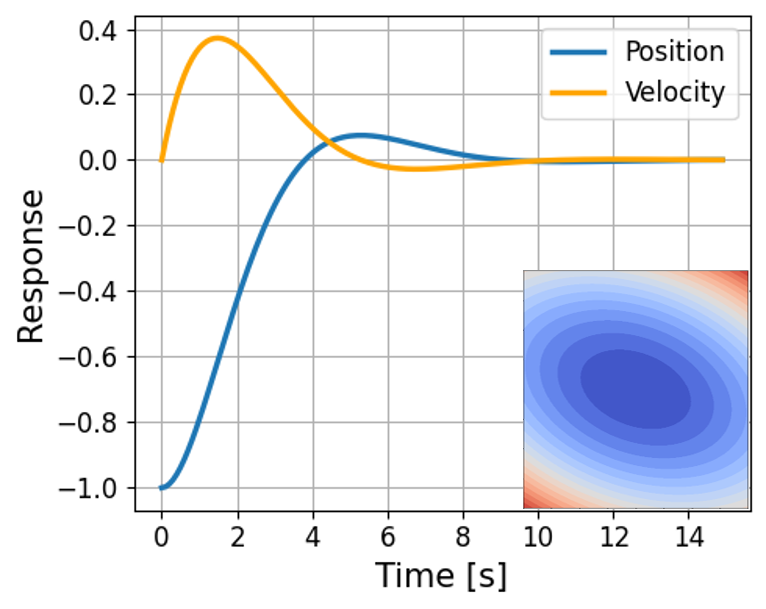}\label{fig:ima_v}}
    \end{minipage}
    \caption{Position and velocity responses of a spring-damper system, along with the approximated value functions (figure insets as colormaps), for the policies derived from the ground truth, EMA, and IMA.}
    \label{fig: spring_damper}
\end{figure*}

\subsection{Theoretical analysis of IMA compared with EMA}

EMA is a commonly used approach for robust policy learning. Given that precise system identification is intractable in real-world scenarios, we present a theoretical proof demonstrating that the control policy obtained by IMA is theoretically superior to that of EMA for practical applications.

\begin{theorem}
    Given an estimate of domain parameter $\bm{\alpha}$, the parameter-conditioned policy derived through implicit motor adaptation is more optimal than that obtained through \textit{explicit motor adaptation}.
\end{theorem}

\begin{proof}
\label{proof: proof3}

Knowing the exact parameter $\bm{\alpha}$ for the real-world domain is intractable, but we can represent it probabilistically. Let ${P}(\hat{\bm{\alpha}})$ denote the probabilistic estimation obtained using any off-the-shelf approach. Therefore, the objective of robust policy learning is to find a policy $\pi$ that maximizes
\begin{equation}
\begin{aligned}
    &J_\pi (\bm{x}, {P}(\hat{\bm{\alpha}})) = \mathbb{E}_{\hat{\bm{\alpha}} \sim {P}(\hat{\bm{\alpha}})} \Bigg[   \mathbb{E}_{\pi}  \Big[ \sum_{t=0}^{\infty} \gamma^t R(\bm{x}_{t}, \bm{u}_{t}) \Big| \bm{x}_{0}\!=\! \bm{x}  \Big] \Bigg], \\
     &\text{s.t.} \quad \bm{x}_{t+1} =  f(\bm{x}_t, \bm{u}_t, \hat{\bm{\alpha}}). 
     \label{eq: ema_obj}
\end{aligned}
\end{equation}

The key difference between IMA and EMA lies in the method used to retrieve the parameter-conditioned policy. EMA determines the domain parameter $\bm{\alpha}$ with a single estimate, defined as $\hat{\bm{\alpha}} = f_\theta (\bm{h})$. The obtained $\hat{\bm{\alpha}}$ is then used as input to the parameter-augmented base policy. The resulting parameter-conditioned policy $\pi_e$ aims to maximize
\begin{equation}
   V(\bm{x}| \hat{\bm{\alpha}}) = V  (\hat{\alpha}_{j_1}, \cdots ,\hat{\alpha}_{j_d}, \bm{x}).
\end{equation}

IMA estimates the true parameter $\bm{\alpha}$ using a probability distribution $\hat{\bm{\alpha}} \sim {P}(\hat{\bm{\alpha}})$, and then retrieves the parameter-conditioned policy $\pi_i$ through domain contraction, with the objective of maximizing
\begin{equation}
   V(\bm{x}|P(\hat{\bm{\alpha}})) = \sum_{j_1=1}^{N_1} \cdots \sum_{j_d=1}^{N_d} P_{(j_1, \cdots, j_d)} V  (\hat{\alpha}_{j_1}, \cdots ,\hat{\alpha}_{j_d}, \bm{x}),
\end{equation}
where $\hat{\bm{\alpha}} \sim {P}(\hat{\bm{\alpha}})$.

Given the same dataset used in EMA and IMA adaptation module training, $\hat{\bm{\alpha}}$ can be considered a single sample from ${P}(\hat{\bm{\alpha}})$. Therefore, we have
\begin{equation}
    J_{\pi_{i}} (\bm{x}, P(\hat{\bm{\alpha}})) \geq  J_{\pi_{e}} (\bm{x}, P(\hat{\bm{\alpha}})).
\end{equation}
This indicates that the policy obtained by EMA typically demonstrates poorer performance compared to that obtained by IMA.
\end{proof}
 
\begin{example}
  Spring-damper System
\end{example}

To numerically illustrate the difference, we provide a spring-damper system with parameters $\bm{\alpha} = (m, k, b)$ as a toy example, where $m$ represents the mass, $k$ the spring stiffness, and $b$ the damping coefficient. The system is described by the following state-space equations
\[
\dot{\bm{x}}_t = \mathbf{A} \bm{x}_t + \mathbf{B} u_t,
\]
where \(\bm{x}_t\) includes the displacement and velocity, and \(u_t\) is the control input. The system matrices are expressed as
\[
\mathbf{A} = \begin{bmatrix} 0 & 1 \\ -\frac{k}{m} & -\frac{b}{m} \end{bmatrix}, \quad
\mathbf{B} = \begin{bmatrix} 0 \\ \frac{1}{m} \end{bmatrix}.
\]

The optimal control law and value function can be obtained by solving the Algebraic Riccati Equation (ARE) \citep{lancaster1995algebraic}, and they are strongly influenced by the parameter values.

We assume that the real system parameters $\bm{\alpha} = [m, k, b]$ are unknown but can be estimated as a probability distribution $\hat{\bm{\alpha}} \sim P(\hat{\bm{\alpha}})$. Fig. \ref{fig: spring_damper} compares the results of EMA and IMA with the ground-truth optimal policy. The value function obtained by IMA closely aligns with the ground truth, effectively stabilizing the mass at the desired equilibrium position. In contrast, the EMA policy shows significantly poorer performance. This toy example highlights the effectiveness of probabilistic system adaptation and implicit policy retrieval in a dynamical system with unknown parameters.


\section{Experimental Results}
\label{sec:result}

We validate the effectiveness of the proposed method on three contact-rich manipulation tasks: \texttt{Hit}, \texttt{Push}, and \texttt{Reorientation}, as shown in Fig.~\ref{fig:setup}. All the primitives are highly dependent on the physical parameters between the object, the robot, and the surroundings. 

\subsection{Base policy learning}
\label{sec:base_policy}

We first learn the parameter-augmented base policy over a range of system parameters using tensor approximation. Since the physical dynamics of \texttt{Hit} are fully known, the control policy can be derived analytically. For \texttt{Push} and \texttt{Reorientation}, TTPI \citep{Shetty24ICLR} is used to compute the parameter-augmented value functions and advantage functions.

\textbf{Hit}: Hitting involves manipulating objects through impact, which is a representative one-shot manipulation task. This means the object can only be hit once, with no additional actions allowed. Determining an appropriate impact is therefore crucial and varies significantly depending on instance-specific domain parameters. In this work, we focus on the planar hitting primitive. The state is the object position, denoted as $\bm{x} = [x, y]$, and the control input is the applied impact $\bm{u} = [I_x, I_y]$. The physical parameters $\bm{\alpha}$ include the object mass $m$ and the friction coefficient $\mu$. 

Given the mass $m$, the friction coefficient $\mu$, the initial state $\bm{x}_0$, and the target $\bm{x}^{\text{des}}$, we define the following single-step reward:
\begin{equation}
\begin{aligned}
	r (\bm{\alpha}, \bm{x}, \bm{u}) 
	&= -\bigl(\|\bm{x} - \bm{x}^{\text{des}}\|^2 \;+\; 0.01\,\|\bm{u}\|^2\bigr),
\end{aligned}
\label{eq:hit_sol}
\end{equation}
where 
\[
\bm{x} 
= \bm{x}_0 
+ \frac{\bm{u}}{m}\, t
- 0.5 \,\frac{\bm{u}}{\|\bm{u}\|}\,\mu\,g \, t^2.
\]
Since the hitting task terminates after a single action, this immediate reward function $r$ effectively serves the same role that an ``advantage function'' would in a multi-step setting, differing only by a constant. We use TT-cross to approximate $r(\bm{\alpha}, \bm{x}, \bm{u})$ across diverse parameter instances in TT format.


\textbf{Push}: Pushing is widely recognized as a challenging task in robot planning and control, primarily due to its hybrid and under-actuated nature \citep{mason1986mechanics, lynch1996stable, mason1999progress}. In this work, we further complicate the problem by considering diverse parameters and aiming to find the instance-aware optimal policy for each. The state of the planar pushing task is characterized by $[s_x, s_y, s_{\theta}, \psi, \phi]$, and the action is represented as $[f_x, f_y, \dot{\psi}, \dot{\phi}]$. Here, $[s_x, s_y, s_{\theta}] \in SE(2)$ denotes the position and orientation of the object in the world frame. $\psi$ is the relative angle of the contact point in the object frame, and $\phi$ represents the distance between the contact point and the object surface. The forces exerted on the object are denoted by $\bm{f} = [f_x, f_y]^\trsp$, while $\bm{v}_p = [\dot{\psi}, \dot{\phi}]^\trsp$ represents the angular and translational velocities of the robot's end-effector. The physical parameters include the object mass $m$, radius $r$, and the friction coefficient $\mu$ between the object and the table.

The robot dynamics is defined based on the quasi-static approximation and the limit surface, resulting in a similar expression as \citep{hogan2020reactive}. To learn the parameter-augmented advantage function, the reward function is defined as
\begin{equation}
	\begin{aligned}
		& r =  -(\rho c_{p} + c_o + 0.01 c_f + 0.01 c_v), \\
	\end{aligned}
\end{equation}
with 
\begin{equation}
	\begin{aligned}
		 c_p &= \lVert \bm{x}_p - \bm{x}_p^\text{des} \rVert/l_p ,\quad c_o = \lVert x_o - x_o^\text{des} \rVert/l_o ,\quad \\
		 c_f &= \lVert \bm{f} \rVert,\quad c_v = \lVert \bm{v}_p \rVert,
	\end{aligned}
\end{equation}
where $\bm{x}_p = [s_x, s_y]$ and $x_o = \theta$ denote the object’s position and orientation. Without loss of generality, we set $\bm{x}^\text{des} = \bm{0}$ as the target configuration. $\bm{f}$ is the control force, while $\bm{v}_p$ is the velocity of the robot's end-effector. $l_p$ and $l_o$ are set to $0.005$ and $0.01 \pi$, respectively.

\textbf{Reorientation}: In this task, we aim to enable the robot to reorient an object using parallel fingers. The state is the orientation angle $\theta$, and the control input is the normal force $f_n$ between the gripper and object. The physical parameters are the object mass $m$, length $l$ and torsional friction coefficient $\mu$. The gravitational torque and normal force $f_n$ are used as braking mechanisms to slow down the object motion. We build the dynamics model of the reorientation primitive based on \citep{vina2016adaptive} as 
\begin{equation}
\begin{aligned}
    I \ddot{\theta} &= \tau_g + 2\tau_f,\\
    \dot{\theta} &= \dot{\theta}_0 - \ddot{\theta} \Delta t,
\end{aligned}
\end{equation}
where $\tau_f = \mu_t f_n^{1+\xi}$ is the torsional sliding friction between robot gripper and the object. In this work, we set $\xi = 0$. $\mu_t$ is the torsional friction coefficient, which is related to the materials and normal force distribution. $\tau_g = mgl \text{sin}(\theta)$ is the gravity torque. We therefore include $\mu_t$, object mass $m$ and length $l$ as the model parameters. The task is to rotate any object from its initial configuration to a vertically upward configuration. To achieve this, the object is given an initial angular velocity $\dot{\theta}_0$ by swinging the robot arm.

The reward function for \texttt{Reorientation} primitive learning is defined as
\begin{equation}
	\begin{aligned}
		& r =  -(\beta c_{g} + c_f), \\
	\end{aligned}
\end{equation}
with $c_g = \lVert {x}_o - {x}_o^\text{des} \rVert$, $c_f = \lVert f_n \rVert$. $x_o$ is the orientation angle $\theta$, and $f_n$ is the normal force between the gripper and object. ${x}_o^\text{des}$ is set to $\pi$ as the reorientation goal, and $\beta$ is set to $10^4$.

In our experiments, we employed an NVIDIA GeForce RTX 3090 GPU with 24GB of memory. The accuracy parameter for cross approximation was set to $\epsilon=10^{-3}$ for TT-Cross approximation. The maximum rank $r_{\max}$ and discount factor were set to 100 and 0.99, respectively. The continuous variables of state, action and parameter domains were discretized as 50 to 500 points using uniform discretization.

\subsection{Results of domain contraction for policy retrieval}
\label{sec:domain_cont}

Given the parameter-augmented advantage functions learned in Section \ref{sec:base_policy}, we can now retrieve the parameter-conditioned policy for each specific instance through domain contraction (DC). This section aims to demonstrate the unification and flexibility of DC compared to domain randomization (DR) and domain adaptation (DA). To illustrate this, we make a simplifying assumption, without loss of generality, that the system adaptation module provides a uniform distribution as the probabilistic representation of parameters. This distribution spans a range defined by the discretization indices of each dimension (denoted as $w$), as shown in Table \ref{tab:reward_eva} and Fig.~\ref{fig:error_eva}. Specifically, $w=1$ indicates that the exact value of the physical parameter is known, corresponding to DA. $w=N$ reflects no prior knowledge about the model parameters, corresponding to DR. These two variants can be seen as extreme cases of DC, where $w$ determines how finely the distribution is utilized for policy retrieval. Nonetheless, our framework is general and flexible, and can accommodate arbitrary parameter distributions, such as those produced by diffusion models \citep{ho2020denoising}.

\renewcommand{\arraystretch}{1.}
\begin{table}[t]
\centering  
\begin{small}
  
	\caption{Cumulative reward of three manipulation tasks}
		\begin{tabular}{l |c |c|c| c|}
			\toprule
			& {$w=1$}	& {$w=N/20$} & {$w=N/5$} & {$w=N$}\\
			\cline{1-5}
			{Hit}  &1.0   &0.65 $\pm$ 0.21  &0.01 $\pm$ 0.01    &0.02 $\pm$ 0.05 \\
			{Push}  &1.0 &0.99 $\pm$ 0.01  &0.99 $\pm$ 0.03 & 0.93 $\pm$ 0.11 \\
			{Reori.}  &1.0   &0.99 $\pm$ 0.04 & 0.99 $\pm$ 0.07 & 0.85 $\pm$ 0.19 \\			
			\bottomrule
		\end{tabular}
	\label{tab:reward_eva}
 \end{small}
 \end{table}


\begin{figure*}[t]
	\centering
	\subfloat[\footnotesize \centering Hit]{{\includegraphics[width=0.65\columnwidth]{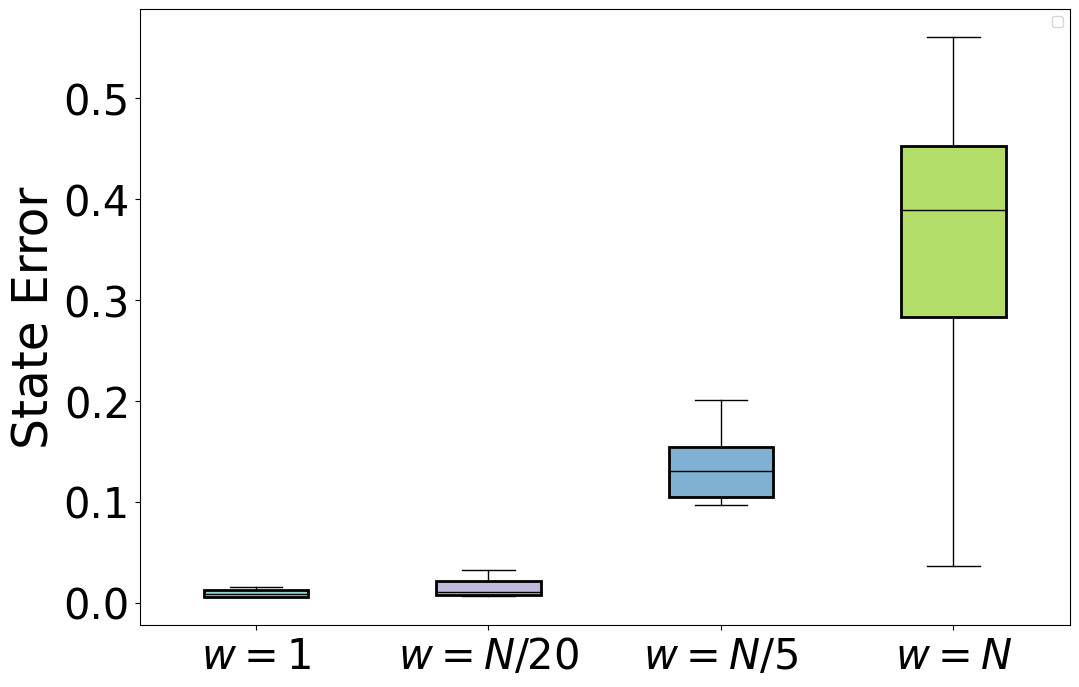}}\label{fig:hit_err}}
	\hfill
	\subfloat[\footnotesize \centering Push]{{\includegraphics[width=0.65\columnwidth]{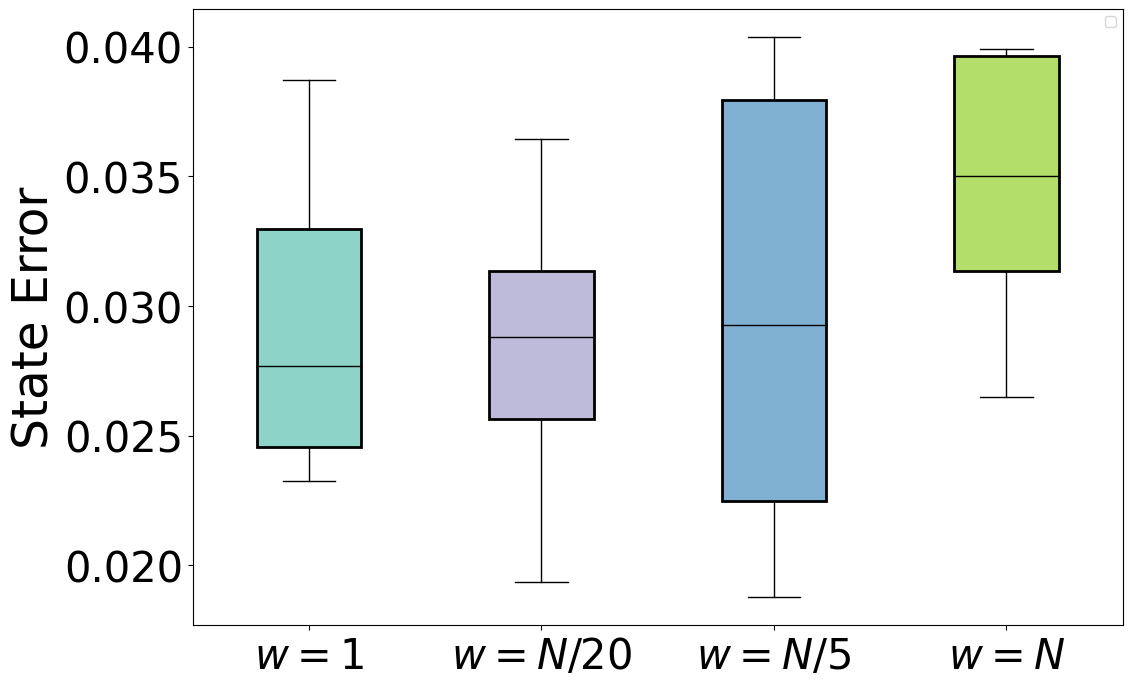}}\label{fig:push_err}}
	\hfill
	\subfloat[\footnotesize \centering Reorientation]{{\includegraphics[width=0.65\columnwidth]{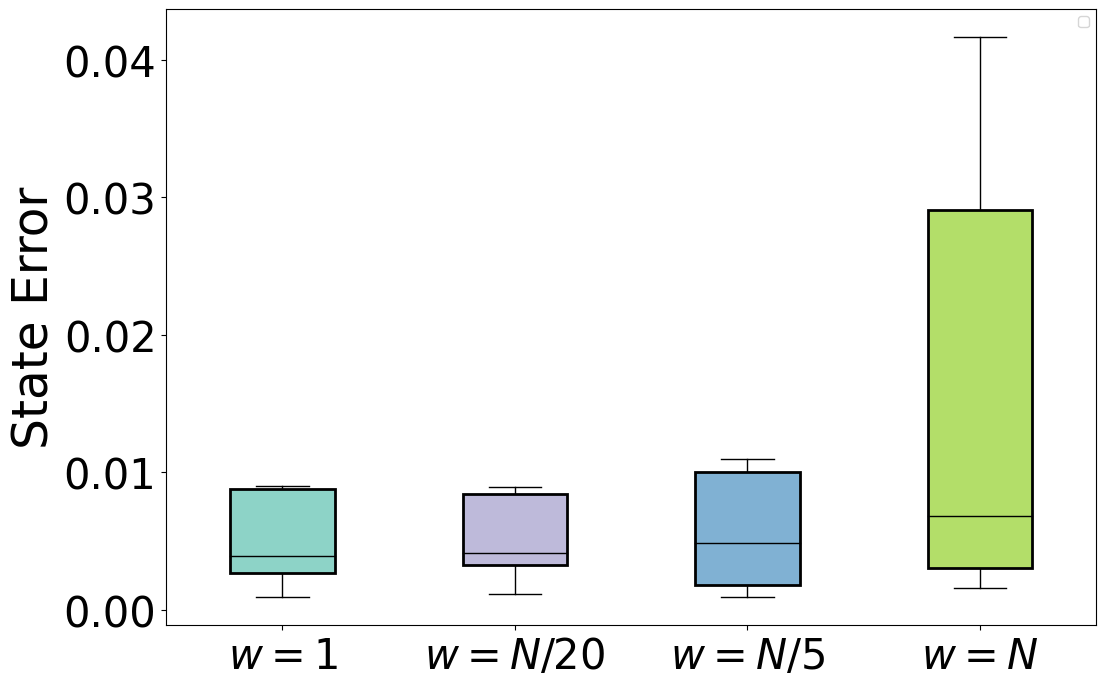}}\label{fig:reori_err}}
	\caption{Comparison of final state error given different estimated parameter distribution.}
	\label{fig:error_eva}
\end{figure*}

Table \ref{tab:reward_eva} and Fig.~\ref{fig:error_eva} demonstrate the comparisons of cumulative reward and final state error, respectively. The state error is quantified as the $\ell_2$ norm of the difference between the final state and the target state. The cumulative reward is normalized using the value obtained through DA. We can observe that DA ($w=1$) typically performs the best in all these three tasks, in particular for the \texttt{Hit} primitive, as it resembles an open-loop control. Once the impact is given from the robot to the object, there is no way to adjust control inputs to influence the object movements further. However, knowing the exactly correct domain parameter is intractable (as shown in Fig.~\ref{fig:param_state_map}). Instead, domain contraction does not require one specific value and accommodate probabilistic estimation. For example, results show that a rough range ($w=N/20$) is sufficient to achieve the target for the \texttt{Hit} primitive. Furthermore, \texttt{Push} and \texttt{Reorientation} can be considered more akin to closed-loop control, which allows for much more rough parameter estimation. As depicted in Table \ref{tab:reward_eva} and Fig.~\ref{fig:error_eva}, a rough distribution with $w=N/5$ is adequate. 

Moreover, based on Table \ref{tab:reward_eva}, we observe that $w=N$ results in the lowest cumulative reward. This is consistent with our assertion that DR typically leads to conservative behaviors. Although DA($w=1$) yields the highest cumulative reward, obtaining precise parameter values is often challenging in the real world. DC bridges the gap between domain adaptation and DR, offering greater flexibility to generate optimal behaviors while leveraging instance-specific rough parameter distributions. This is much practical for real-world contact-rich manipulation tasks.

\subsection{Results of implicit v.s. explicit motor adaptation}
\label{sec:ma_sysID}

\begin{figure}[t]
	\centering
	\includegraphics[width=0.35\textwidth]{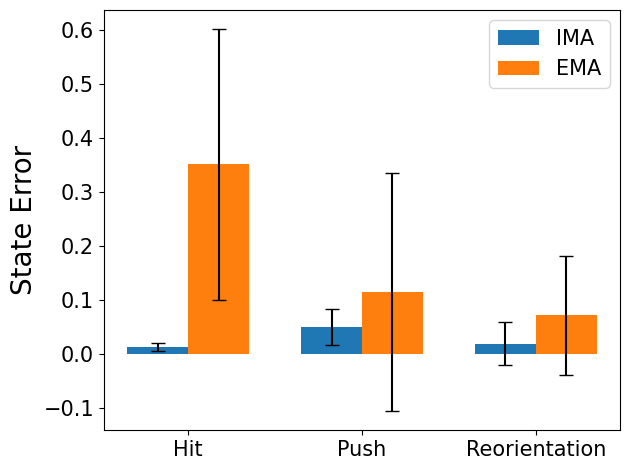}
	\caption{Comparison between IMA and EMA on three manipulation primitives in terms of the final state error, given the same base policy.}
	\label{fig:ema_ima_err}
\end{figure}

\begin{figure}[t]
	\centering
	\includegraphics[width=0.46\textwidth]{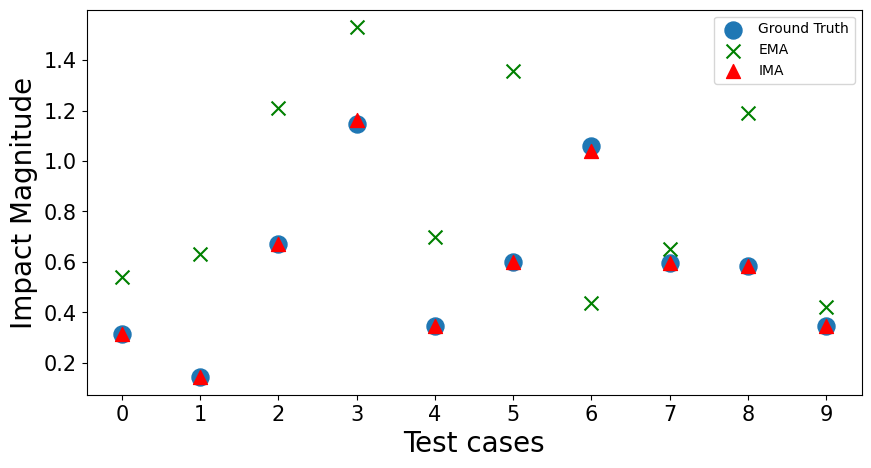}
	\caption{Resulting hitting impact of the policies derived from the ground truth, EMA, and IMA.}
	\label{fig:hit_impact}
\end{figure} 

In this section, we compare the performance of IMA and EMA on each task with 10 different instances, by randomizing system parameters and initial conditions while using the same base policy. As shown in Fig.~\ref{fig:ema_ima_err}, IMA generally outperforms EMA on the three contact-rich manipulation tasks, particularly in the open-loop hitting task. In this paper, we employ a simple MLP with a structure of $512 \times 256 \times 128 \times 64$ for parameter estimation. EMA directly uses the output (denoted as $\hat{\bm{\alpha}}_t$) to retrieve the parameter-conditioned policy from the base policy. In contrast, IMA assumes $\hat{\bm{\alpha}}_t$ as the mean $\bm{\nu}_t$ of a uniform distribution $\mathbb{U}(\bm{\nu}_t - \bm{w}/2, \bm{\nu}_t + \bm{w}/2)$, where the range $w$ is set to $N/20$ for \texttt{Hit} and $N/5$ for \texttt{Push} and \texttt{Reorientation}. 

Fig. \ref{fig:hit_impact} illustrates the computed impacts for 10 hitting instances with different physical parameters, where the ground truth values are represented by blue dots. By utilizing a probabilistic representation of the estimated parameters and domain contraction for policy retrieval, IMA typically produces values closely aligned with the ground truth, whereas EMA often yields suboptimal results. We further analyze the performance of EMA and IMA under varying levels of parameter estimation accuracy, as shown in Fig.~\ref{fig:param_state_map}, and study the impact of different choices for $w$ in DC. When parameters are precisely estimated, EMA achieves the lowest final state error, demonstrating optimal performance. However, as parameter estimation becomes less accurate, greater randomization (a larger $w$) is required to obtain a better policy. The dotted line and gray shaded area in the figure indicate the estimation accuracy achieved in our work, with IMA ($w = N/3$) providing the best results. This analysis highlights the effectiveness of IMA in real-world scenarios where system parameters are unknown and can only be roughly estimated. It also demonstrates that domain randomization ($w = N$) is not always good when some domain knowledge is available, even if imperfect. We believe IMA offers a promising approach to find control policies when we have limited access to the domain knowledge.

Figures \ref{fig:push_ima_ema} and \ref{fig:reori_ima_ema} further illustrate the performance of IMA and EMA in the pushing and reorientation tasks, respectively. In Fig.~\ref{fig:push_ima_ema}, the same base policy is used to push a sugar box and a mustard bottle toward their target configurations, including both position and orientation. While both approaches eventually reach the target positions, IMA outperforms EMA in terms of orientation. This is primarily due to the underactuated dynamics of the planar pushing task, where system parameters such as the friction coefficient and object mass play a critical role. IMA estimates these parameters probabilistically and retrieves parameter-conditioned policies implicitly. This allows for increased tolerance of typical system identification errors, enabling robust contact-rich manipulation with unknown parameters. Similarly, Fig.~\ref{fig:reori_ima_ema} demonstrates the superior performance of IMA over EMA in reorienting an arbitrary object vertically from the bottom to the top.

\begin{figure}[t]
	\centering
	\includegraphics[width=0.4\textwidth]{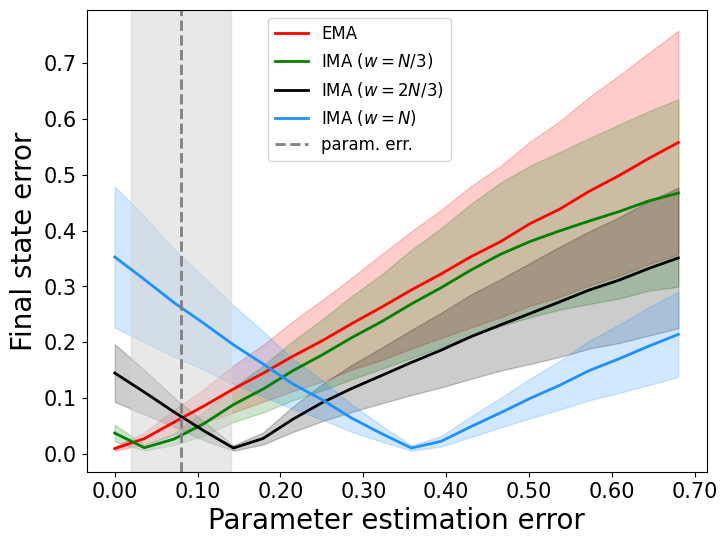}
	\caption{Resulting final state error of different motor adaptation strategies under varying levels of parameter estimation accuracy. The dotted line and gray shaded area represent the mean and standard deviation of the parameter estimation error achieved using the system estimator in this article.}
	\label{fig:param_state_map}
\end{figure}

\begin{figure}[t]
    \centering
    \begin{minipage}{0.24\textwidth}
        \centering
        \includegraphics[width=\linewidth]{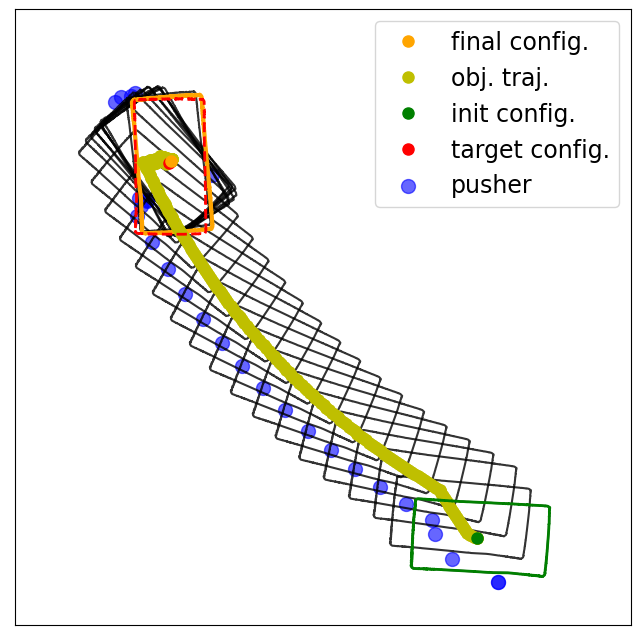}
    \end{minipage}
    \hfill
    \begin{minipage}{0.24\textwidth}
        \centering
        \includegraphics[width=\linewidth]{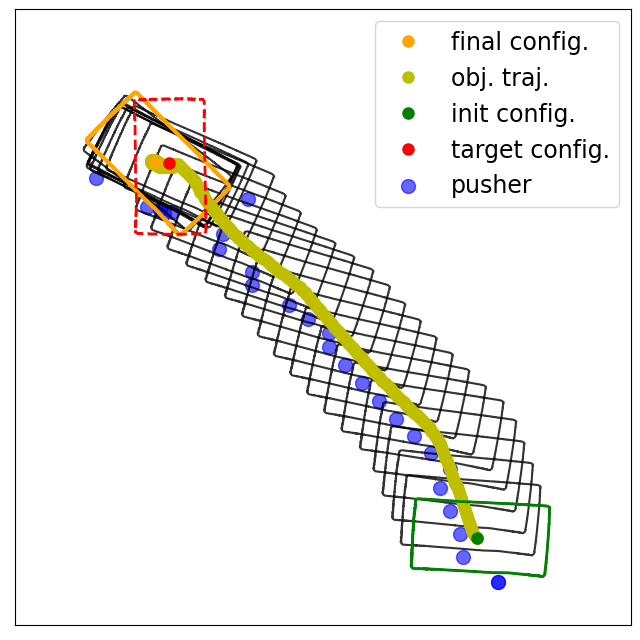}
    \end{minipage}
    
    \vspace{0.5cm} 

    \begin{minipage}{0.24\textwidth}
        \centering
        \includegraphics[width=\linewidth]{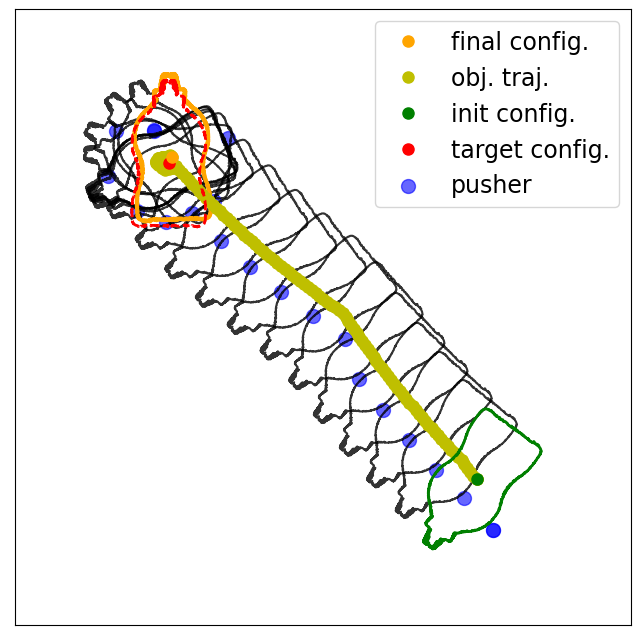}
    \end{minipage}
    \hfill
    \begin{minipage}{0.24\textwidth}
        \centering
        \includegraphics[width=\linewidth]{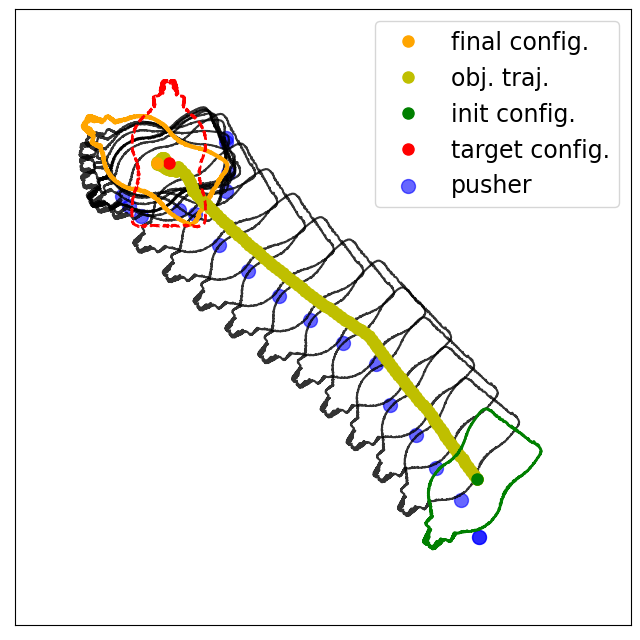}
    \end{minipage}

    \caption{Planar pushing tasks with a sugar box (\emph{Top}) and a mustard bottle (\emph{Bottom}). \textbf{Left}: Object trajectory produced by IMA; \textbf{Right}: Object trajectory produced by EMA.}
    \label{fig:push_ima_ema}
\end{figure}

\begin{figure}[htbp]
    \centering
    \begin{minipage}{0.45\linewidth}
        \centering
        {\includegraphics[width=0.95\linewidth]{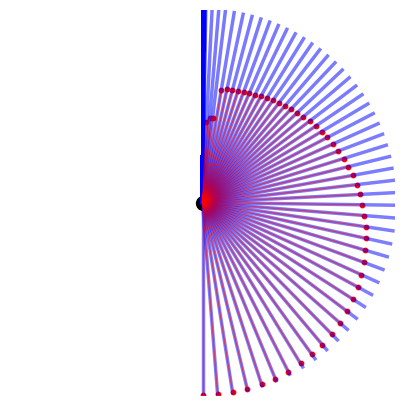}\label{fig:ima_ori}}
    \end{minipage}
    \hfill
    \begin{minipage}{0.45\linewidth}
        \centering
        {\includegraphics[width=0.95\linewidth]{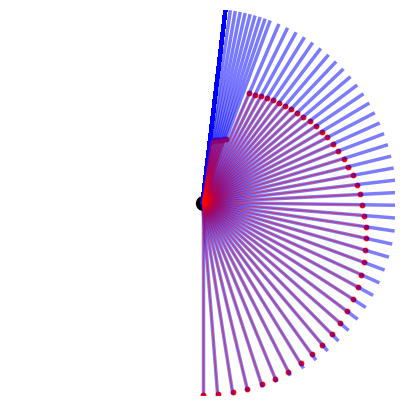}\label{fig:ema_reori}}
    \end{minipage}
    \caption{Trajectories produced by IMA (\textbf{Left}) and EMA (\textbf{Right}) for the reorientation task. The blue lines depict the trajectory of an arbitrary object, with the objective of reorienting it from the bottom configuration to a perfectly vertical top configuration. The lengths of the red lines indicate the angular velocity magnitude, which is indirectly controlled by leveraging gravity and friction forces.}
    \label{fig:reori_ima_ema}
\end{figure}

\begin{figure*}[t]
	\centering
	{{\includegraphics[width=0.65\columnwidth]{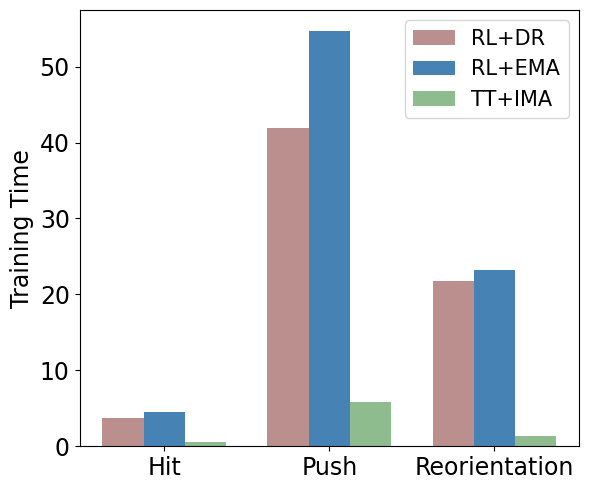}}\label{fig:trainingtime_comp}}
	\hfill
	{{\includegraphics[width=0.65\columnwidth]{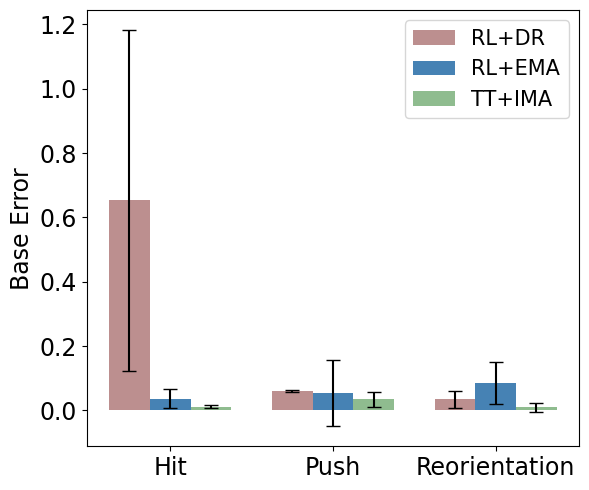}}\label{fig:base_err_com}}
	\hfill
	{{\includegraphics[width=0.65\columnwidth]{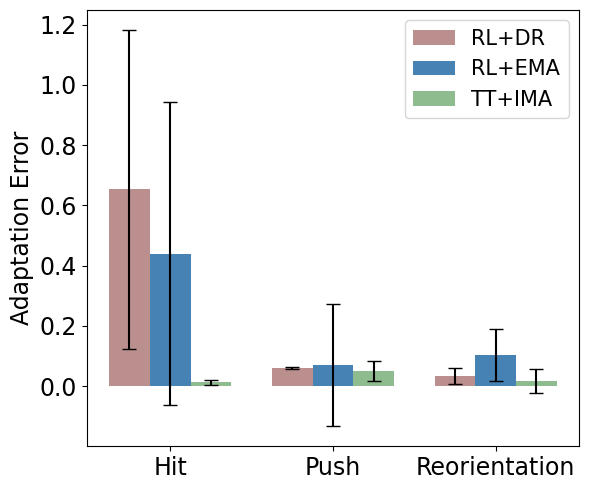}}\label{fig:adapt_err_comp}}
	\caption{Comparison of different robust primitive learning approaches in terms of training time, and state errors of base policies and adapted policies. The unit of training time is seconds for \texttt{Hit} and minutes for \texttt{Push} and \texttt{Reorientation}. Errors are calculated based on the $\ell_2$ norm between the final configuration and the target configuration.}
	\label{fig:comp_robust_mp}
\end{figure*}

\begin{figure*}[t] 
	\centering
	\subfloat[\footnotesize \centering Initialization]{{\includegraphics[width=0.5\columnwidth]{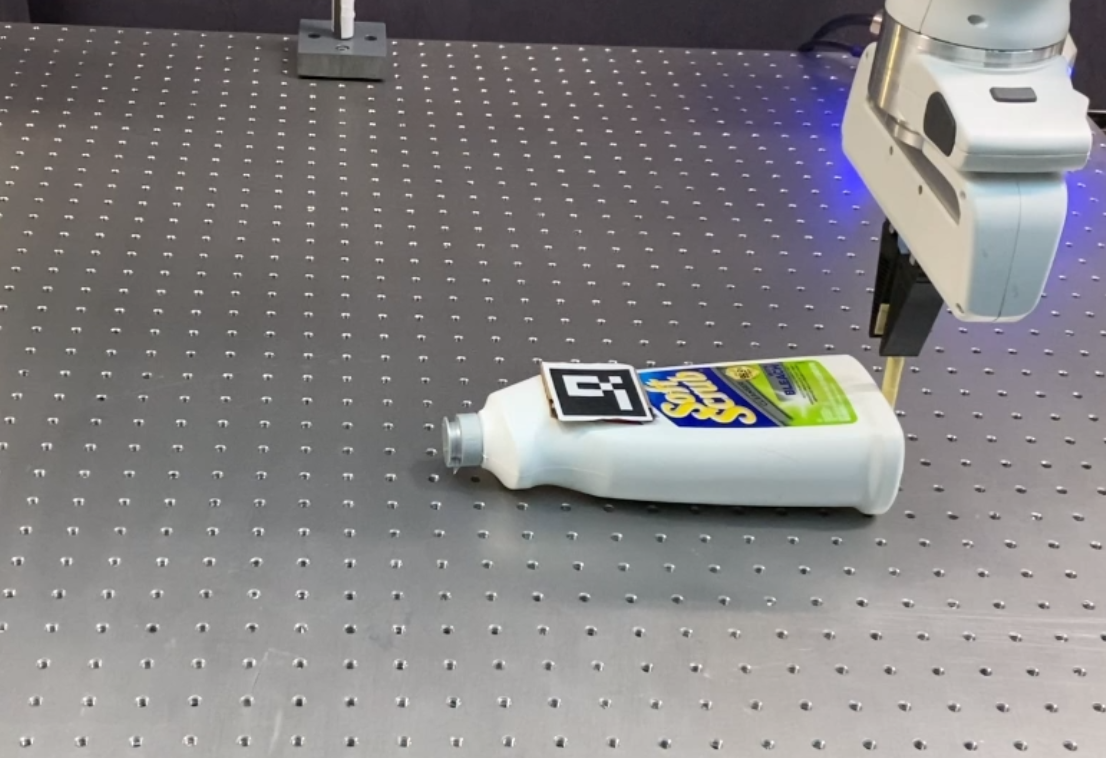}}\label{fig:robot_init}}
	\subfloat[\footnotesize \centering Pushing]{{\includegraphics[width=0.5\columnwidth]{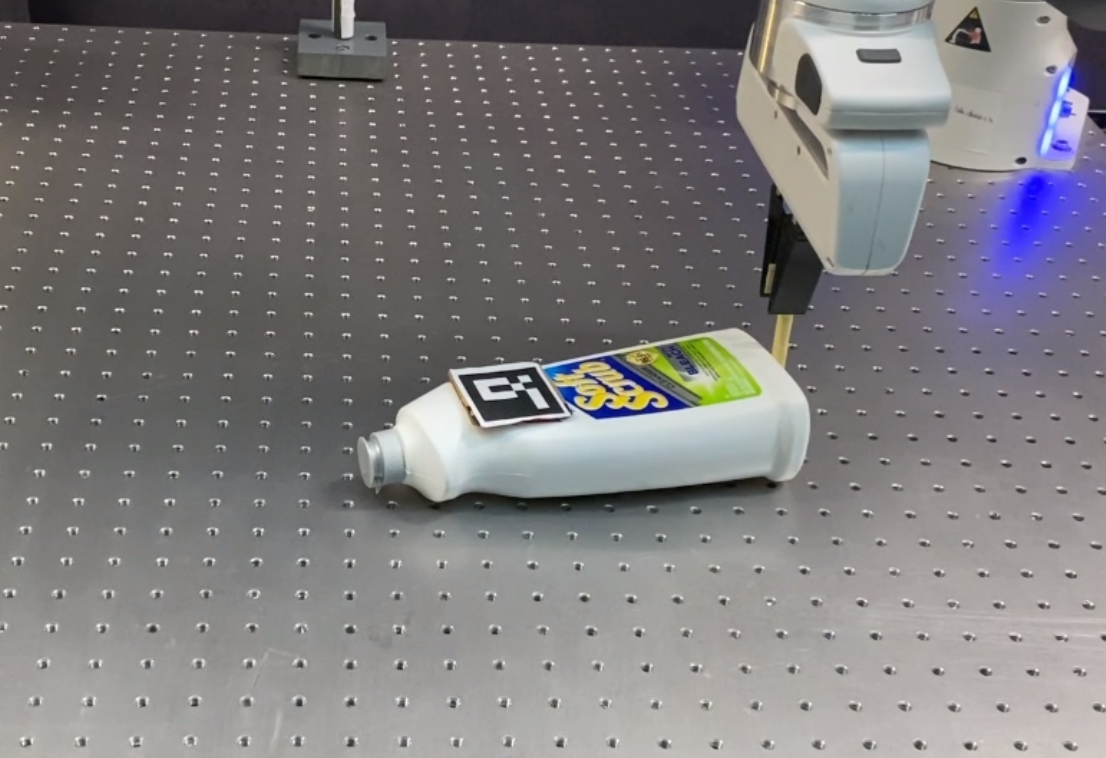}}} 
	\subfloat[\footnotesize \centering Contact switching]{{\includegraphics[width=0.5\columnwidth]{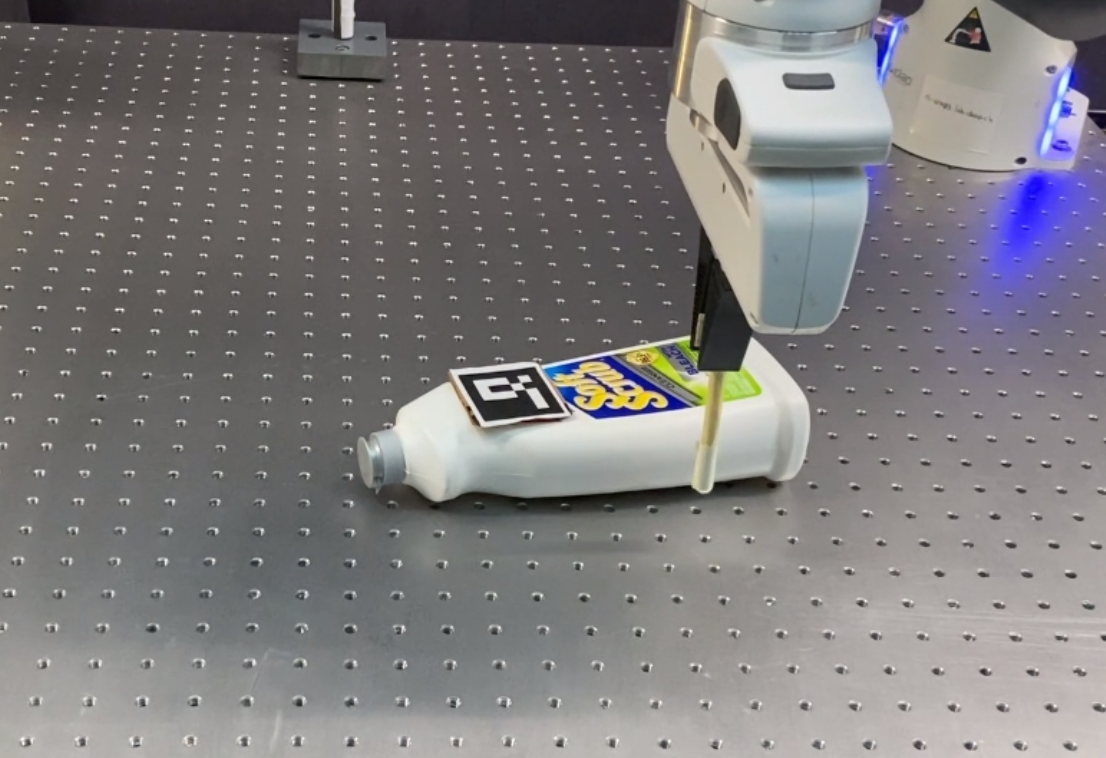}}} 
		\subfloat[\footnotesize \centering Disturbance]{{\includegraphics[width=0.5\columnwidth]{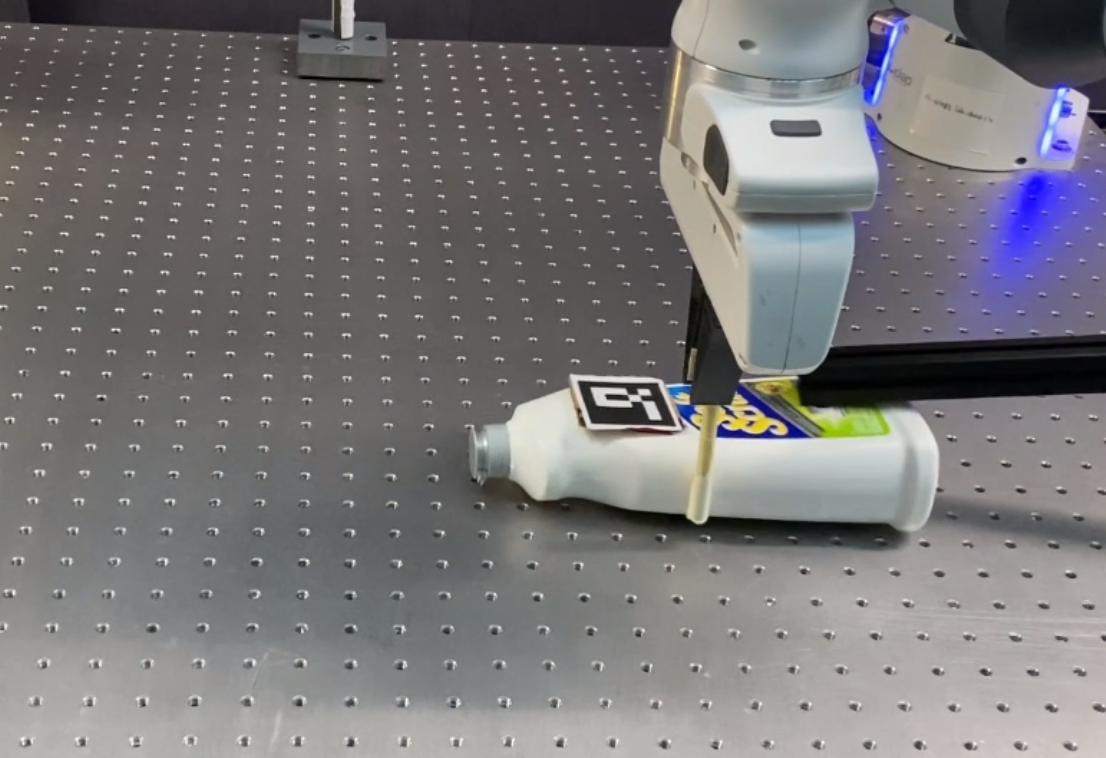}}} \\
	\subfloat[\footnotesize \centering Contact switching]{{\includegraphics[width=0.5\columnwidth]{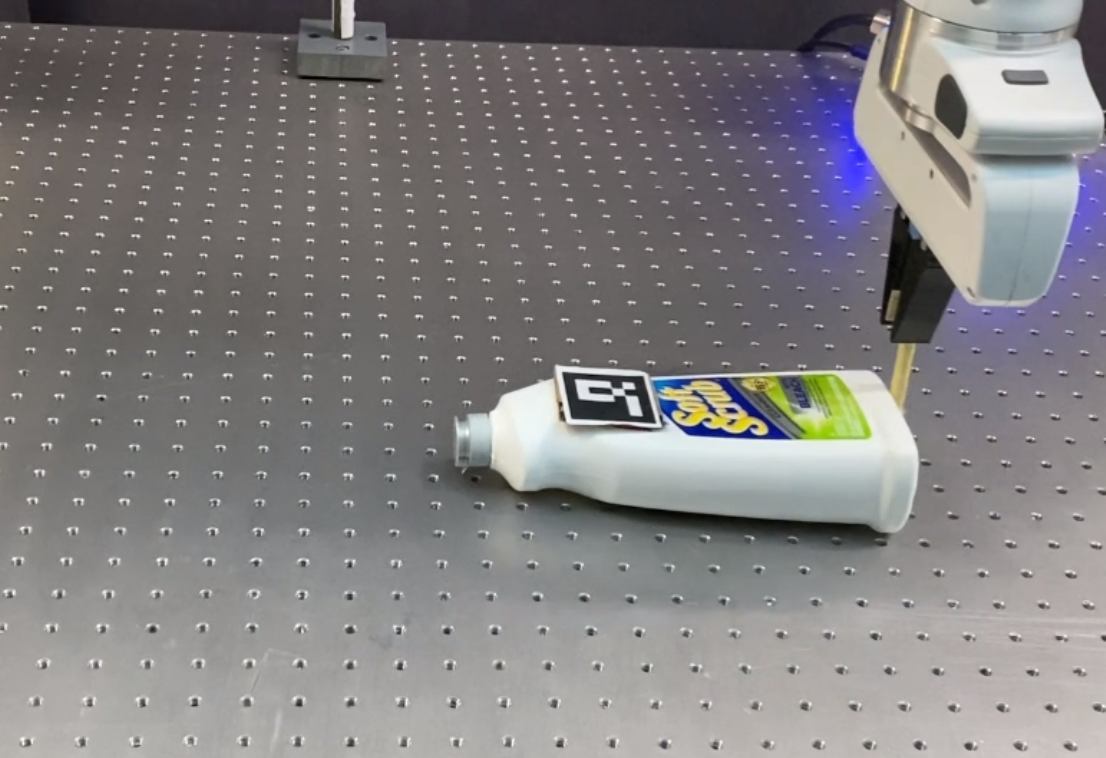}}} 
	\subfloat[\footnotesize \centering Pushing]{{\includegraphics[width=0.5\columnwidth]{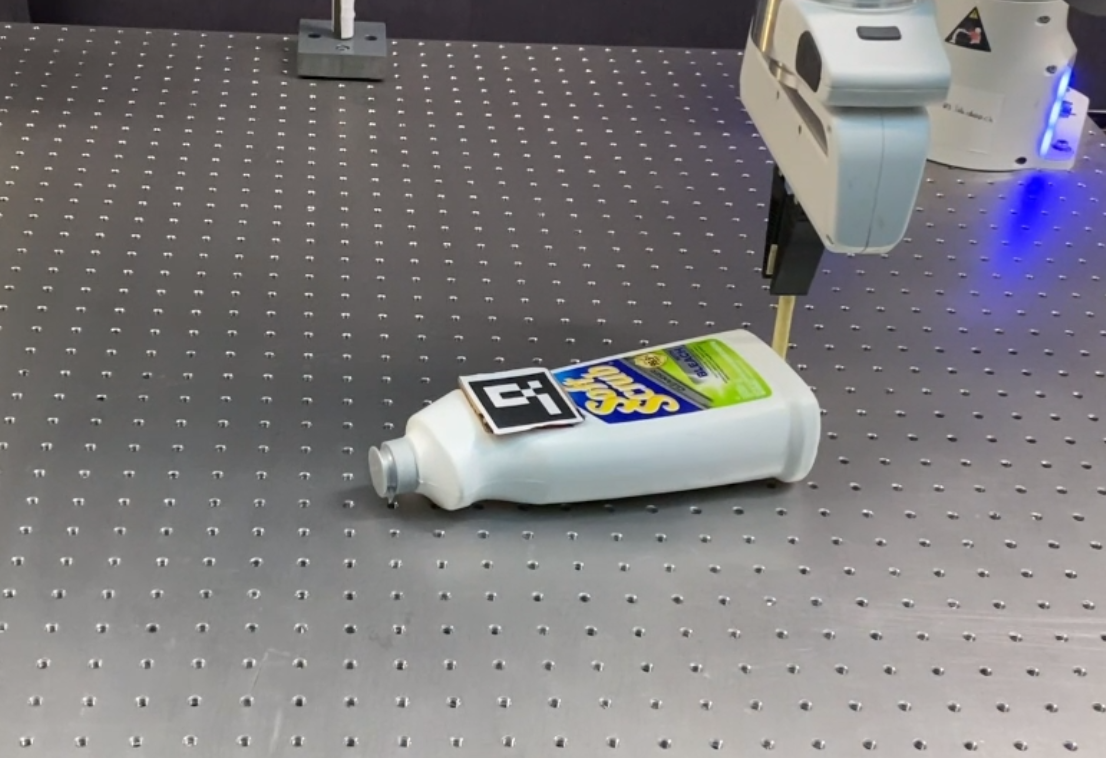}}} 
	\subfloat[\footnotesize \centering Contact switching]{{\includegraphics[width=0.5\columnwidth]{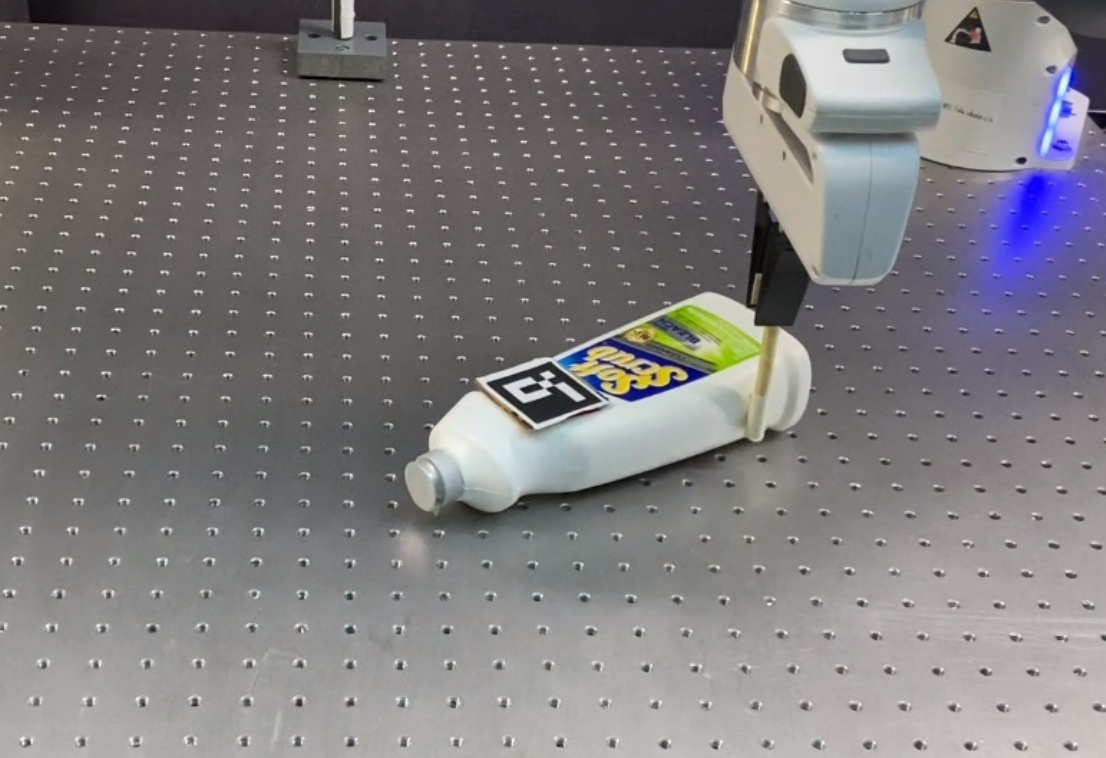}}} 
	\subfloat[\footnotesize \centering Target Reaching]{{\includegraphics[width=0.5\columnwidth]{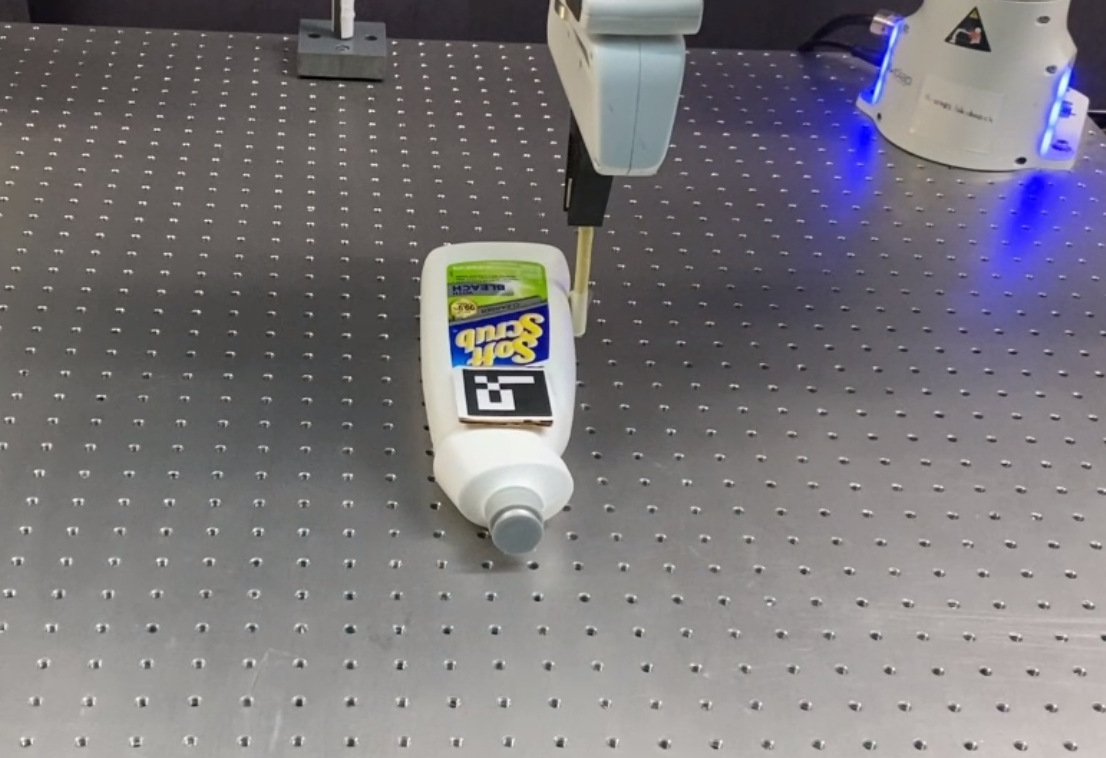}}\label{fig:robot_end}}
	\caption{An example of planar pushing task: bleach bottle on a metal surface.  The system begins in the initial configuration (a), aiming to manipulate the bottle to reach the target configuration (h). The robot first makes contact with the bottle and pushes it slightly (b), then switches the contact point (c). When an external disturbance (d) is introduced by a human, the robot adapts by switching the contact point again (e) to accommodate the unexpected change. It continues pushing the bottle further (f), performs another contact switch (g), and ultimately reaches the target configuration (h).}
	\label{fig:keyframe}
	\vspace{-0.5cm}
\end{figure*}

\subsection{Results of robust manipulation primitive learning}

We also compared our full pipeline with two widely used robust policy learning approaches: reinforcement learning with domain randomization (RL+DR) \citep{tobin2017domain} and reinforcement learning with explicit motor adaptation (RL+EMA) \citep{yu2017preparing, qi2023hand}. We utilize Soft Actor-Critic (SAC) \citep{haarnoja2018soft} for policy learning in both RL+DR and RL+EMA, as SAC has been shown to be the state-of-the-art RL algorithm for such contact-rich manipulation tasks, especially planar pushing, as demonstrated in \citep{potters2022learning, cong2022reinforcement}. Our implementation is based on Stable-Baselines3 \citep{stable-baselines3}, using a Multilayer Perceptron (MLP) with a 64 $\times$ 64 $\times$ 64 architecture as the policy network. The discount factor is set to 0.99, and the learning rate to 0.001.

The quantitative results are presented in Fig.~\ref{fig:comp_robust_mp}, including the time required for policy training and the state error of the retrieved policy. The \textit{base\_error} and \textit{adaptation\_error} refer to the final state error of the base policy with the true parameter and the parameter-conditioned policy after motor adaptation, respectively. For RL+DR, the \textit{base\_error} and \textit{adaptation\_error} are the same since no motor adaptation is performed.

Notably, although the \texttt{Hit} task is a one-shot manipulation task in which proprioceptive history is unavailable, the motor adaptation module works by observing the trajectories of past rollouts, similarly to how humans perform such tasks. The results show that our method requires significantly less time for base policy training compared to the other methods, while achieving the lowest \textit{base\_error}, highlighting the advantages of leveraging TT for learning control policies in contact-rich manipulation tasks.

Furthermore, our method (TT+IMA) achieves the lowest final state error among the three approaches, demonstrating the overall effectiveness of our complete approach for robust contact-rich manipulation tasks, with implicit action representation playing a crucial role in retrieving robust policies. An additional observation is that RL+EMA often exhibits the highest standard deviation, indicating that its control policies perform extremely poorly in some instances, whereas IMA has a better statistical performance over the diverse instances.


\renewcommand{\arraystretch}{1.}
\begin{table}[t]
\centering   
\begin{small}
 
	\caption{Time required for computing parameter-conditioned advantage function using core-level or function-level operations}
		\begin{tabular}{l |c |c |c|}
			\toprule
			& {core-level}	& {function-level}\\
			\cline{1-3}
			{Hit} &0.016s $\pm$ 0.002s   & 0.720s $\pm$ 0.236s   \\
			{Push}  & 0.075s $\pm$ 0.003s &18.56s $\pm$ 4.748s   \\
			{Reorientation} & 0.018s $\pm$ 0.003s   & 8.494s $\pm$ 0.561s  \\			
			\bottomrule
		\end{tabular}
	\label{tab:time}
\end{small}
\end{table}

\subsection{Sensitivity analysis and ablation study}
\label{sec:ablation}

Our framework involves several hyperparameters, such as the maximum TT-rank, the discretization granularity, and the bandwidth of the uniform distribution. To analyze their impact on policy performance, we conduct a sensitivity analysis across a range of values for each hyperparameter. Figure~\ref{fig:sensitivity_analysis} shows how policy performance, measured by the $\ell_2$ error between the final and target states, varies with the maximum TT-rank and discretization granularity. For comparability across tasks, the errors are normalized by the worst-case performance in each task.

From Figure~\ref{fig:sensitivity_analysis}, we observe that setting the maximum rank $r_{\text{max}}$ above 20 leads to consistently low errors across all tasks. For \texttt{Push} and \texttt{Hit}, even lower ranks (e.g., $r_{\text{max}} = 10$ or $5$) are sufficient, indicating that the underlying policies lie in a low-rank functional space. A relatively low tensor rank is sufficient to capture the essential structure, and policy performance remains stable when the rank is increased further. Empirically, we observe that the piecewise-smooth dynamics \citep{toussaint2020describing} commonly encountered in contact-rich manipulation tasks induce structured, low-dimensional variations in the state-action space. The corresponding advantage functions often exhibit a low-rank structure, which can be efficiently captured using the TT representation.


Similarly, the policy is robust to the choice of discretization granularity. As shown in Figure~\ref{fig:sensitivity_analysis}, when the number of discretization intervals exceeds 8, performance does not change much with further refinement, indicating insensitivity to this hyperparameter. This robustness can be attributed to the fact that low-rank TT models provide smooth function approximations, which support effective interpolation between discretized points (as discussed in Section~\ref{sec:function_apprx}). While extremely coarse discretization may degrade performance, in practice, setting the number of intervals to 10 or 20 generally yields stable and accurate results.

For the influence of the bandwidth for the uniform distribution used in probabilistic domain adaptation, we refer the readers to Section~\ref{sec:domain_cont} for details. In summary, a coarse bandwidth (\( w = N/5 \), where \( N \) is the number of discretization points) is sufficient for closed-loop tasks such as \texttt{Push} and \texttt{Reorientation}. In contrast, \texttt{Hit}, as a one-shot (open-loop) task, benefits from a more precise distribution, with \( w = N/20 \) being adequate. In general, the bandwidth parameter for other contact-rich tasks can be determined similarly based on whether the task is open-loop or closed-loop, without requiring significant tuning effort.

To further demonstrate the efficiency of the TT representation, we compared it against non-TT function approximation methods, particularly Neural Networks (NN). In our framework, both TT and NN can be used to learn parameter-augmented base policies and to approximate parameter-augmented advantage functions. The key question is which method offers greater efficiency in the subsequent steps: computing the parameter-conditioned advantage function and retrieving the corresponding policy via an $\arg\max$ operation.

\begin{figure}[htbp]
	\centering
	\includegraphics[width=0.45\textwidth]{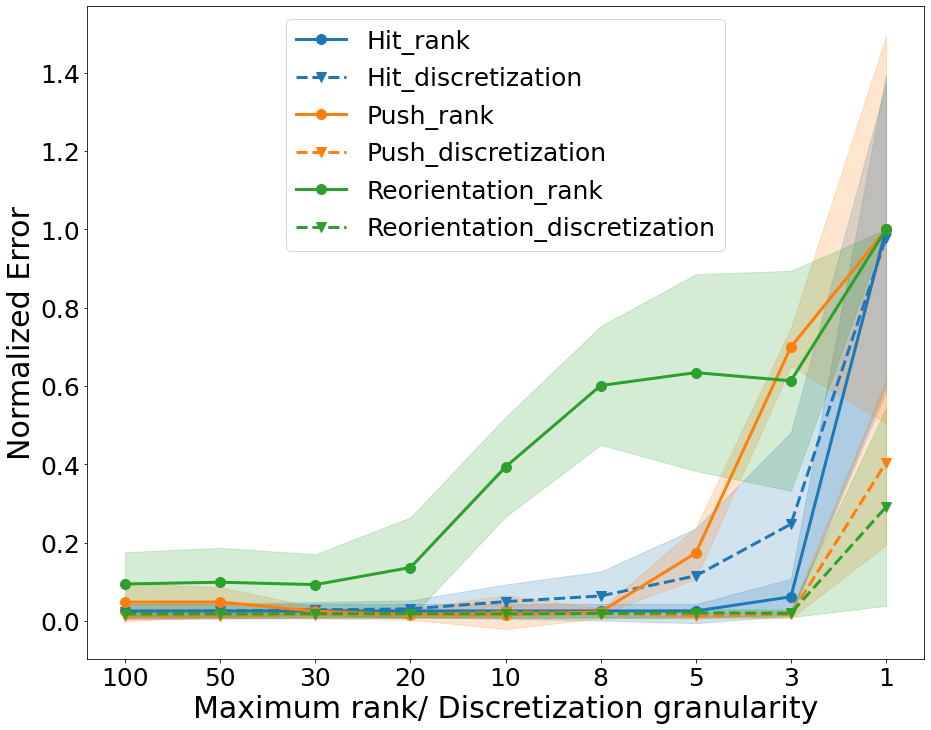}
	\caption{Sensitivity analysis of policy performance with respect to the maximum TT-rank (solid lines) and discretization granularity (dashed lines). The y-axis shows the normalized $\ell_2$ error between the final and target states, scaled by the worst-case performance in each task. The x-axis denotes the maximum TT-rank and the number of discretization intervals (non-uniformly spaced to enhance readability in the low-rank and coarse-discretization regions). The results demonstrate that policy performance is not highly sensitive to these hyperparameters.}
	\label{fig:sensitivity_analysis}
\end{figure}

\begin{table*}[htbp]
\centering   
\begin{small}
	\caption{Comparison of TT and NN representation for parameter-conditioned policy retrieval. 
	Error refers to the task execution error using the retrieved policy, while Opt. Time indicates the time required to perform the $\argmax$ operation.}
		\begin{tabular}{l |c c|c c| c c|}
			\toprule
			& \multicolumn{2}{c|}{TT Optimization} & \multicolumn{2}{c|}{TT Refinement} & \multicolumn{2}{c|}{NN Optimization} \\
			\cline{2-7}
			& {Task Error} & {Opt. Time (s)} & {Task Error} & {Opt. Time (s)}  & {Task Error} & {Opt. Time (s)} \\
			\midrule
			{Hit} & 0.010 $\pm$ 0.008 & 0.003 $\pm$ 0.001 & 0.002 $\pm$ 0.001 & 0.015 $\pm$ 0.001 & 0.002 $\pm$ 0.004 & 0.346 $\pm$ 0.021 \\
			{Push} & 0.034 $\pm$ 0.024 & 0.013 $\pm$ 0.001 & 0.033 $\pm$ 0.024 & 0.146 $\pm$ 0.001 & 1.469 $\pm$ 0.574 & 2.020 $\pm$ 0.002 \\
			{Reorientation} & 0.061 $\pm$ 0.104 & 0.002 $\pm$ 0.001 & 0.061 $\pm$ 0.104 & 0.062 $\pm$ 0.001 & 0.751 $\pm$ 0.680 & 1.514 $\pm$ 0.092 \\
			\bottomrule
		\end{tabular}
	\label{tab:ablation_nn_tt}
\end{small}
\end{table*}

Table~\ref{tab:time} highlights the significant computational advantage of TT in computing the advantage function. Specifically, TT models support efficient \emph{core-level} operations, which avoid combinatorial computation across the multi-dimensional parameter space, leading to a complexity of \(\mathcal{O}(Ndr^2)\), where \(N\) is the number of discretization points per dimension, \(d\) is the dimensionality of the parameter space, and \(r\) is the TT rank, which typically remains small in practice, as demonstrated in our experiments. In contrast, NN-based methods (or any non-TT models) rely on \emph{function-level} operations, resulting in a complexity of \(\mathcal{O}(N^{d})\), which makes the computation substantially slower.


After computing the parameter-conditioned advantage function, the next step is policy retrieval via an $\argmax$ operation. As shown in Table~\ref{tab:ablation_nn_tt}, TT again demonstrates superior efficiency. Direct TT-based optimization (TT Optimization) yields low-error policies with minimal computation time, achieving near-global solutions within milliseconds. In contrast, NN-based approaches (NN Optimization) typically rely on gradient-based solvers. In this work, we specifically adopt the Adam optimizer~\citep{kingma2014adam} as a representative method for comparison. The results indicate that this approach is slower and more susceptible to suboptimal convergence, particularly in complex tasks such as \texttt{Push} and \texttt{Reorientation}.

To further improve accuracy, we also evaluated a hybrid approach (TT Refinement), which uses TT optimization for initialization, followed by a gradient-based Newton-type method. While this refinement slightly improves performance, the gain is often marginal compared to the additional computational cost. In the tasks considered in this work, TT optimization alone is sufficient to achieve high-quality results. These findings highlight the strong potential of TT representation for efficient and robust policy learning in complex, contact-rich manipulation tasks.

\subsection{Real-robot experiments: planar push}
\label{sec:real_exp}

We validated the proposed method for the planar pushing task using a 7-axis Franka robot and a RealSense D435 camera. The manipulated objects included a sugar box and a bleach cleanser from the YCB dataset \citep{calli2015ycb}, each with different shapes and masses, as shown in Fig.~\ref{fig:setup}. The friction coefficients between the objects and the table were varied by using a metal surface and plywood, respectively. Note that it is easier to control the robot kinematically rather than using force control. We therefore leveraged the ellipsoidal limit surface to convert the applied force to velocity, resulting in the motion equations shown in \citep{hogan2020feedback, xue2023guided}. We trained a parameter-augmented policy in simulation and then applied it in the real world through domain contraction. The experiments demonstrate the effectiveness of the obtained parameter-conditioned policies in manipulating instances with diverse parameters. Additionally, external disturbances were introduced by humans, showcasing the reactivity of the retrieved policy. Figure \ref{fig:keyframe} presents keyframes of the learned planar pushing primitive applied to a bleach bottle on a metal surface, exemplifying a specific case within a variety of diverse pushing scenarios. Further results are demonstrated in the accompanying video.

\section{Conclusion and Future Work}
\label{sec:conclusion}

In this article, we propose \textit{implicit motor adaptation} for robust manipulation primitive learning, which enables parameter-conditioned policy retrieval given a probabilistic parameter distribution rather than a single estimate. We prove that, to achieve this, the policy must be represented implicitly as the $\arg\max$ of the advantage function, rather than treated as an explicit feed-forward function. The state-value and advantage functions are represented in tensor train (TT) format, facilitating efficient policy retrieval through operations at the tensor core level instead of at the function level. Our approach differs from the well-known \textit{explicit motor adaptation} framework, which either requires precise system identification or additional student policy training. Theoretical analysis and numerical results demonstrate the superior performance of \textit{implicit motor adaptation} over \textit{explicit motor adaptation}. Simulation and real-world experiments further validate the effectiveness of our approach in contact-rich manipulation tasks under parameter uncertainty and external disturbances.

Our work paves the way for robust sim-to-real transfer. It can be easily integrated with more general policy learning approaches, such as reinforcement learning methods and classical approaches (e.g., LQR), provided the base policy can be implicitly represented with a state-action value function. The key insight is to represent the base policy implicitly and retrieve the parameter-conditioned policy probabilistically. This concept also holds promise for achieving robust behavior cloning, where the policy can be represented through energy functions (as in implicit behavior cloning \citep{florence2022implicit}) or denoising functions (as in diffusion policy \citep{chi2023diffusion}).

In this article, we used TT to represent the value function and the advantage function, both obtained via TT-Cross approximation \citep{oseledets2010_ttcross1, usvyatsov2022tntorch}. Although TT-Cross enables efficient global approximation by actively querying function values, it struggles to scale to very high-dimensional settings. In the future, we plan to combine TT with neural networks in a data-driven manner \citep{dolgov2023data, shetty24thesis} to address this limitation.

Additionally, we employed a simple MLP for system identification, assuming that the parameters follow a uniform distribution. This approach could be extended by adopting more advanced models, such as diffusion models \citep{ho2020denoising}. Large Visual-Language Models \citep{gao2024physically} could also be explored to leverage their commonsense understanding of domain knowledge from scenario images.

\section*{Acknowledgments}
This work was supported by the China Scholarship Council (grant No.202106230104), and by the State Secretariat for Education, Research and Innovation in Switzerland for participation in the European Commission's Horizon Europe Program through the INTELLIMAN project (\url{https://intelliman-project.eu/}, HORIZON-CL4-Digital-Emerging Grant 101070136) and the SESTOSENSO project (\url{http://sestosenso.eu/}, HORIZON-CL4-Digital-Emerging Grant 101070310). We thank Jiacheng Qiu for suggestions about the implementation of RL baselines.

\bibliographystyle{plainnat}
\bibliography{references}





\end{document}